\documentclass{article}

\usepackage{microtype}
\usepackage{graphicx}
\usepackage{subfigure}
\usepackage{booktabs} 
\usepackage{multicol,lipsum}

\usepackage{hyperref}

\usepackage{url}
\usepackage{wrapfig}
\usepackage{float}
\usepackage{amsmath,amssymb,amsthm}

\usepackage{graphicx}
\graphicspath{ {images/} }
\RequirePackage{filecontents}
\usepackage{bibunits}
\usepackage[usenames, dvipsnames]{color}

\newcommand{\pth}[1]{\left( #1\right)}                 
\newcommand{\brk}[1]{\left[ #1\right]}                 
\newcommand{\braces}[1]{\left\lbrace #1\right\rbrace } 
\newcommand{\norm}[1]{\left\lVert#1\right\rVert}       
\newcommand{\normal}[1]{\mathcal{N}\pth{#1}}
\newcommand{\E}{\mathbb{E}} 
\newcommand{\Excpt}[2]{\underset{#1\sim #2\,\,}{\E}}

\newcommand{\KLsymbol}{D}  
\newcommand{\KL}[2]{\KLsymbol{\pth{#1||#2}}}

\newcommand{\diag}{\mathop{\mathrm{diag}}}
\newcommand{\Grad}[1]{\nabla_{#1}}

\newcommand{\isEquivTo}[1]{\underset{#1}{\sim}}
\newcommand{\Prob}{\mathbb{P}}
\newcommand{\Union}[2]{\bigcup_{#1}^{#2}}
\newcommand{\Intersect}[2]{\bigcap_{#1}^{#2}}

\DeclareMathOperator*{\argmin}{argmin}

\newtheorem{Theorem}{Theorem}
\newtheorem{Lemma}{Lemma}

\newcommand{\er}[1][Q]{\textit{er} \pth{#1}}

\newcommand{\erhat}[1][Q]{\widehat{\textit{er}} \pth{#1}}

\newcommand{\erhati}[2]{\widehat{\textit{er}}_{#2} \pth{#1}}

\newcommand{\refeq}[1]{{(\ref{#1})}}

\newcommand{\Qcal}{\mathcal{Q}}
\newcommand{\Pcal}{\mathcal{P}}

\newcommand{\loss}[1]{\ell \pth{#1}}

\newcommand{\thetatild}{{\tilde{\theta}}}
\newcommand{\Dcal}{{\cal D }}

\usepackage[accepted]{icml2018}

\icmltitlerunning{Meta-Learning by Adjusting Priors Based on Extended PAC-Bayes Theory}

\defaultbibliography{bib_file}
\defaultbibliographystyle{icml2018}

\begin{document}
\begin{bibunit} 

\twocolumn[
\icmltitle{Meta-Learning by Adjusting Priors \\ Based on Extended PAC-Bayes Theory}



\icmlsetsymbol{equal}{*}

\begin{icmlauthorlist}
\icmlauthor{Ron Amit}{Technion}
\icmlauthor{Ron Meir}{Technion}

\end{icmlauthorlist}

\icmlaffiliation{Technion}{The Viterbi Faculty of Electrical Engineering, Technion - Israel Institute of Technology, Haifa, Israel}

\icmlcorrespondingauthor{Ron Amit}{ronamit@campus.technion.ac.il}
\icmlcorrespondingauthor{Ron Meir}{rmeir@ee.technion.ac.il}

\icmlkeywords{Machine Learning, ICML, Meta Learning, Transfer Learning, PAC-Bayes}

\vskip 0.3in
]



\printAffiliationsAndNotice{}  


\begin{abstract}
In meta-learning an agent extracts knowledge from observed tasks, aiming to facilitate learning of novel future tasks. Under the assumption that future tasks are `related’' to previous tasks,  the accumulated knowledge should be learned in a way which captures the common structure across learned tasks, while allowing the learner sufficient flexibility to adapt to novel aspects of new tasks. We present a framework for meta-learning that is based on generalization error bounds, allowing us to extend various PAC-Bayes bounds to meta-learning. Learning takes place through the construction of a distribution over hypotheses based on the observed tasks, and its utilization for learning a new task. Thus, prior knowledge is incorporated through setting an experience-dependent prior for novel tasks.
We develop a gradient-based algorithm which minimizes an objective function derived from the bounds and demonstrate its effectiveness numerically with deep neural networks.
In addition to establishing the improved performance available through meta-learning, we demonstrate the intuitive way by which prior information is manifested at different levels of the network.

\end{abstract}
\section{Introduction}
Learning from examples is the process of inferring a general rule from a finite set of examples. It is well known in statistics (e.g., \cite{DevGyoLug96}) that learning cannot take place without prior assumptions. 
Recent work in deep neural networks has achieved significant success in using prior knowledge in the implementation of structural constraints, e.g., convolutions and weight sharing \citep{lecun2015deep}. However, often the relevant prior information for a given task is not clear, and there is a need for building it through learning from previous interactions with the world. Learning from previous experience can take several forms: \textbf{Continual Learning} \citep{kirkpatrick2017overcoming} - a learning agent is trained on a sequence of tasks, aiming to solve the current task while maintaining good performance on previous tasks. 
\textbf{Multi-Task Learning} \citep{caruana1998multitask} - a learning agent learns how to solve several observed tasks, while exploiting their shared structure.
\textbf{Domain Adaptation} \citep{ben2010theory} - a learning agent  solves a `target’' learning task using `source’' tasks (both are observed, but usually the target is predominantly unlabeled).
We work within the framework of  \textbf{ Meta-Learning / Learning-to-Learn / Inductive Transfer } \citep{Thrun-1997-14508, vilalta2002perspective}
\footnote{In our setting all observed tasks are available simultaneously to the meta-learner. The setting in which task are observed sequentially is often termed as \textbf{Lifelong Learning }\citep{thrun1996learning,alquier2017regret}}
in which a `meta-learner’' extracts knowledge from several observed tasks to facilitate the learning of \emph{new tasks} by a `base-learner’' (see Figure \ref{fig:Diagram}). In this setup the meta-learner must generalize from a finite set of observed tasks. The performance is evaluated when learning related new tasks (which are unavailable to the meta-learner) .

As a motivational example, consider the case in which a meta-learner observes many image classification tasks of natural images, and uses a CNN to learn each task. The meta-learner might learn a prior which fixes the lower layers of the network to extract generic image features, but allows variation in the higher layers to adapt to new classes. Thus, new tasks can be learned using fewer examples than learning from scratch (e.g., \cite{yosinski2014transferable}).
Generally, other scenarios might instead benefit from sharing other parts of the network  (e.g., \cite{yin2017knowledge}). In our framework the prior is automatically inferred from the observed tasks, rather than being manually inserted by the algorithm designer.


The notion of `task-environment' was formulated by \cite{baxter2000model}. In analogy to the standard single-task learning where data is sampled from an unknown distribution, Baxter suggested a setting where tasks are sampled from an unknown task distribution (environment), so that knowledge acquired from previous tasks can be used in order to improve performance on a novel task. Baxter's work not only provided an interesting and mathematically precise perspective for meta-learning, but also provided generalization bounds demonstrating the potential improvement in performance due to prior knowledge. 

In this paper we work within the framework formulated by \cite{baxter2000model}, and, following the setup in \cite{pentina2014pac}, provide generalization error bounds within the PAC-Bayes framework. These bounds are then used to develop a practical learning algorithm that is applied to neural networks, demonstrating the utility of our approach. 
The main contributions of this work are the following. 
\textit{(i)} An improved and tighter bound in the theoretical framework of \cite{pentina2014pac} derived using a technique which can extend different single-task PAC-Bayes bounds to the meta-learning setup.
\textit{(ii)}  A principled meta-learning method and its implementation using probabilistic feedforward neural networks. 
\textit{(iii)} Empirical demonstration of the performance enhancement compared to naive approaches as well as recent methods in this field. 

\begin{figure}[t]
	\includegraphics[scale=0.25]{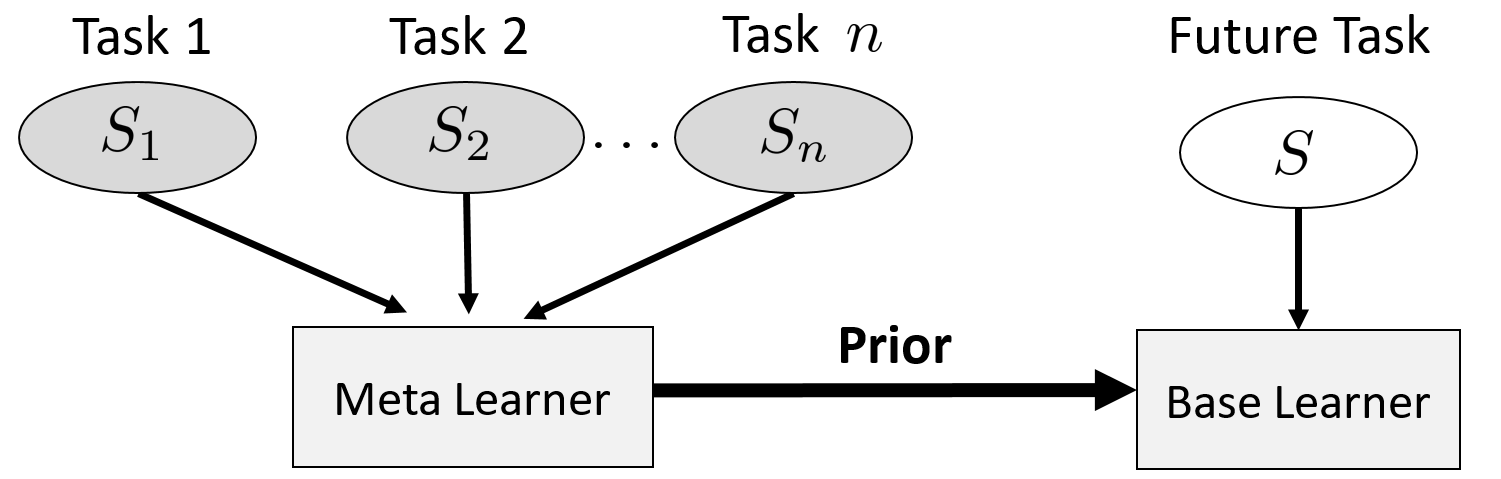}
	\centering
	\caption{The meta-learner uses the data sets of the observed tasks $S_1,...,S_n$ to infer `prior knowledge'  which in turn can facilitate learning in future tasks from the task-environment (which are unobserved by the meta-learner).}
	\label{fig:Diagram}
\end{figure}

\paragraph{Related Work} 
While there have been many recent developments in meta-learning (e.g., \citet{edwards2016towards,andrychowicz2016learning,finn2017model}), most of them were not based on generalization error bounds, which is the focus of the present work. An elegant extension of generalization error bounds to meta-learning was provided by \citet{pentina2014pac}, mentioned above (extended in \citet{pentina2015lifelong}). Their work, however, did not provide a practical algorithm applicable to deep neural networks. More recently, \citet{dziugaite2017computing} developed a
single-task algorithm based on PAC-Bayes bounds that was demonstrated to yield good performance in simple classification tasks with deep networks. Other recent theoretical approaches to meta or multitask learning (e.g.~\citet{maurer2005algorithmic, maurer2009transfer, ruvolo2013ella, maurer2016benefit, alquier2017regret}) provide increasingly general bounds but have not led to practical algorithms for neural networks.  

\section{Preliminaries: PAC-Bayes Learning} \label{sect:Preliminaries}

In the common setting for learning, a set of independent samples, $S=\{z_i\}_{i=1}^m$, from a space of examples $\mathcal{Z}$, is given, each sample drawn from an \emph{unknown} probability distribution $\Dcal$, namely $z_i\sim \Dcal$.
 We will use the notation $S\sim \Dcal^m$ to denote the distribution over the full sample. 
In supervised learning, the samples are input/output pairs $z_i=(x_i,y_i)$.
The usual learning goal is, based on $S$,  to find a hypothesis $h\in\mathcal{H} $,  where $\cal H$ is the so-called hypothesis space, that minimizes the expected loss function $\mathbb{E}\ell(h,z)$, where $\ell(h,z)$ is a \emph{loss function} bounded in $[0,1]$ .
As the distribution $\Dcal$ is unknown, learning consists of selecting an appropriate $h$ based on the sample $S$. In classification $\mathcal{H}$ is a space of classifiers mapping the input space to a finite set of classes.  As noted in the Introduction, an inductive bias is required for effective learning.
While in the standard approach to learning, described in the previous paragraph, one usually selects a single classifier (e.g., the one minimizing the empirical error), the PAC-Bayes framework, first formulated by \citet{mcallester1999pac}, considers the construction of a complete probability distribution over $\mathcal{H}$, and the selection of a single hypothesis $h\in\mathcal{H}$ based on this distribution. Since this distribution depends on the data it is referred to as a \textit{posterior distribution} and will be denoted by $Q$. We note that while the term `posterior' has a Bayesian connotation, the framework is not necessarily Bayesian, and the posterior does not need to be related to the prior through the likelihood function as in standard Bayesian analysis. 

The PAC-Bayes framework has been widely studied in recent years, and has given rise to significant flexibility in learning, and, more importantly, to some of the best generalization bounds available \cite{seeger2002pac,  catoni2007pac, audibert2010pac,lever2013tighter}.
{Recent works analyzed transfer-learning in neural networks with PAC-Bayes tools \citep{galanti2016theoretical,mcnamara2017risk}}.
The framework has been recently extended to the meta-learning setting by \cite{pentina2014pac}, and will be extended and applied to neural networks in the present contribution.	

\subsection{Single-task Problem Formulation} \label{sect:SingleTaskDef}
Following the notation introduced above we define the expected error $\er[h,\Dcal] \triangleq \Excpt{z}{\Dcal} \ell(h,z)$ and the empirical error $\erhat[h,S] \triangleq (1/m)\sum_{j=1}^{m}\loss{h, z_i}$ \rm{} {for a single hypothesis $h \in \mathcal{H}$}. Since the distribution $\Dcal$ is unknown, $\er[h,\Dcal]$ cannot be directly computed. 
In the PAC-Bayes setting the learner outputs a distribution over the entire hypothesis space $\mathcal{H}$,  i.e, the goal is to provide a \textit{posterior} distribution $Q \in \mathcal{M}$, where  $\mathcal{M}$ denotes the set of distributions over  $\mathcal{H}$. The \textit{expected error} and \textit{empirical error} are then given in this setting by averaging  over the posterior distribution $Q \in \mathcal{M}$, namely 
$\er[Q,\Dcal] \triangleq \Excpt{h}{Q} \er[h,\Dcal]$ and 
$\erhat[Q, S] \triangleq \Excpt{h}{Q} \erhat[h,S]$, respectively.

\subsection{PAC-Bayes Generalization Bound} \label{sect:SingleTaskBound}
In this section we introduce a PAC-Bayes bound for the single-task setting.
The bound will also serve us for the meta-learning setting in the next sections. 
PAC-Bayes bounds are based on specifying some `prior' reference distribution $P \in \mathcal{M}$,  that must not depend on the observed data $S$.
The distribution over hypotheses $Q$ which is provided as an output from the learning process is called the posterior (since it is allowed to depend on $S$)
\footnote{As noted above, the terms `prior' and `posterior' might be misleading, since, this is not a Bayesian inference setting (the prior and posterior are not connected through the Bayes rule). However, PAC-Bayes and Bayesian analysis have interesting and practical connections, as we will see in the next sections (see also \citet{germain2016pac}).}.	
%
The classical PAC-Bayes theorem for single-task learning was formulated by \citet{mcallester1999pac}. 
\begin{Theorem}[McAllester's single-task bound]   
\label{thm:OriginalPacBayes}
	Let $P \in \mathcal{M}$ be some prior distribution over $\mathcal{H}$.
	Then for any $\delta \in (0,1]$, the following inequality holds uniformly for all posteriors distributions $Q \in \mathcal{M}$ with probability at least $1-\delta$,
	\begin{equation*} 
	\er[Q,\Dcal] \leq 	\erhat[Q,S] + \sqrt{\frac{ \KL{Q}{P} + \log \frac{m}{\delta}}{2(m-1)}}
	\end{equation*}
\end{Theorem}

where $\KL{Q}{P}$ is the  Kullback-Leibler (KL) divergence,
$
\KL{Q}{P} \triangleq \Excpt{h}{Q} \log \frac{Q(h)}{P(h)}.
$


Theorem \ref{thm:OriginalPacBayes} can be interpreted as stating that with high probability the expected error $\er[Q,\Dcal]$ is upper bounded by the empirical error plus a complexity term. Since, with high probability, the bound holds uniformly for all $Q \in \mathcal{M}$, it holds also for data dependent $Q$. By choosing $Q$ that minimizes the bound we obtain a learning algorithm with generalization guarantees. Note that PAC-Bayes bounds express a trade-off between fitting the data (empirical error) and a complexity/regularization term (distance from prior) which encourages selecting a `simple' hypothesis, namely one similar to the prior. 
The specific choice  of $P$ affects the bound's tightness and so should express  prior knowledge about the problem. Generally, we want the prior to be close to posteriors which can achieve low training error. 



\section{PAC-Bayes Meta-Learning}
In this section we introduce the meta-learning setting.  In this setting a meta-learning agent observes several `training' tasks from the same task environment.
The meta-learner must extract some common knowledge (`learned prior') from these tasks, which will be used for learning new tasks from the same environment.
In the literature this setting is often called learning-to-learn, lifelong-learning, meta-learning or bias learning \citep{baxter2000model}. We will formulate the problem and provide a generalization bound which will later lead to a practical algorithm. Our work extends \citet{pentina2014pac} and establishes a potentially tighter bound.
Furthermore, we will demonstrate how to apply this result practically to deep neural networks using stochastic learning. 

\subsection{Meta-Learning Problem Formulation} \label{sect:MetaProblemDef}
The meta-learning problem formulation follows \citet{pentina2014pac}.
We assume all tasks share the sample space $\mathcal{Z}$, hypothesis space $\mathcal{H}$ and loss function $\ell: \mathcal{H}\times \mathcal{Z} \to \brk{0,1}$.
The learning tasks differ in the unknown sample distribution $\Dcal_t$ associated with each task $t$.
The meta-learning agent observes the training sets $S_1,...,S_n$  corresponding to $n$ different tasks.	The number of samples in task $i$ is denoted by $m_i$.
Each observed dataset $S_i$ is assumed to be generated from an unknown sample distribution  $S_i \sim \Dcal^{m_i}_i$. As in \citet{baxter2000model}, we assume that the sample distributions $\Dcal_i$ are generated $i.i.d.$~from an unknown tasks distribution $\tau$.
The goal of the meta-learner is to extract some knowledge from the observed tasks that will be used as prior knowledge for learning new (yet unobserved) tasks from $\tau$. The prior knowledge comes in the form of a distribution over hypotheses, $P \in \mathcal{M}$. When learning a new task, the \textit{base learner} uses the observed task's data $S$ and the prior $P$ to output a posterior distribution $Q(S,P)$ over $\mathcal{H}$.
We assume that all tasks are learned via the same learning process. Namely, for  a given $S$ and $P$ there is a specific output $Q(S,P)$.
Hence the base learner $Q$ is a mapping:  $Q: \mathcal{Z}^{m} \times \mathcal{M} \rightarrow  \mathcal{M}.$
\footnote{In the next section we will use stochastic optimization methods as learning algorithms, but we can assume convergence to a same solution for any execution with a given $S$ and $P$.}

The quality of a prior $P$ is measured by the expected loss when using it to learn new tasks, as defined by,
\begin{equation} 
\er[P, \tau] \triangleq \Excpt{(\Dcal,m)}{\tau} \Excpt{S}{\Dcal^m} \Excpt{h}{Q(S,P)} \Excpt{z}{\Dcal}  \ell(h,z).
\end{equation}
The expectation is taken w.r.t.~\textit{(i)} tasks drawn from the task environment, \textit{(ii)} training samples, \textit{(iii)} hypotheses drawn from the posterior which is learned based on the training samples and prior \textit{(iv)} a `test' sample.

As described in section \ref{sect:Preliminaries}, in the single-task PAC-Bayes framework, the learner assumes a prior over hypotheses $P(h)$, then observes the training samples and outputs a posterior distribution over hypotheses $Q(h)$.
In an analogous way, in the meta-learning PAC-Bayes framework, the meta-learner assumes a prior distribution over priors, a `hyper-prior' $\Pcal(P)$,  observes the training tasks, and then outputs a distribution over priors, a `hyper-posterior' $\Qcal(P)$.

We emphasize that while the hyper-posterior is learned using the observed tasks, the goal is to use it for learning new, independent task from the environment.
When encountering a new task, the learner samples a prior from the hyper posterior $\Qcal(P)$, and then use it for learning. 
Ideally, the performance of the hyper-posterior $\Qcal$  is measured by the expected loss of learning new tasks using priors drawn from $\Qcal$.  
This quantity is denoted as the \textit{transfer error}
\begin{equation} \label{def:TransferRisk}
\er[\Qcal, \tau]  \triangleq \Excpt{P}{\Qcal} \er[P, \tau].
\end{equation}
While $\er[\Qcal, \tau]$ is not computable,  we can however evaluate the average empirical risk when learning the observed tasks using priors drawn from $\Qcal$, which is denoted as the \textit{empirical multi-task error}
\begin{equation} \label{eq:EmpricalMultiTaskRisk}
\erhat[\Qcal,S_1,...,S_n] \triangleq \Excpt{P}{\Qcal} \frac{1}{n}\sum_{i=1}^n \erhat[Q(S_i,P),S_i],
\end{equation} 

{In the single-task PAC-Bayes setting one selects a prior $P\in \mathcal{M}$ before seeing the data, and updates it to a posterior $Q\in \mathcal{M}$ after observing the training data. In the present meta-learning setup, following \citet{pentina2014pac}, one selects an initial hyper-prior distribution $\mathcal{P}$, essentially a distribution over prior distributions $P$, and, following the observation of the data from all tasks, updates it to a hyper-posterior distribution $\mathcal{Q}$.
As a simple example, assume the initial prior $P$ is a Gaussian distribution over neural network weights, characterized by a mean and covariance. A hyper distribution would correspond in this case to a distribution over the mean and covariance of $P$.} 

\subsection{Meta-Learning PAC-Bayes Bound} \label{sect:MetaBound}
In this section we present a novel bound on the transfer error in the meta-learning setup.
The theorem is proved in section \ref{sect:MetaBoundProof} of the supplementary material.

\begin{Theorem}[Meta-learning PAC-Bayes bound] \label{thm:MetaBound}
	Let $Q: \mathcal{Z}^{m} \times \mathcal{M} \rightarrow  \mathcal{M}$ be a base learner, and let $\Pcal$ be some predefined hyper-prior distribution.
	Then for any $\delta \in (0,1]$ the following inequality holds uniformly for all hyper-posterior distributions $\Qcal$ with probability  at least $1-\delta$, \footnote{The probability is taken over sampling of $(\Dcal_i, m_i) \sim\tau$ and  $S_i \sim  \Dcal_i^{m_i},i=1,...,n$.}
	\begin{gather}  \label{eq:MetaBound}
	\er[\Qcal, \tau] \leq 
\frac{1}{n} \sum_{i=1}^{n} \Excpt{P}{\Qcal}  \erhati{Q_i,S_i}{i}   \\
+\frac{1}{n} \sum_{i=1}^{n}  \sqrt{\frac{ \KL{\Qcal}{\Pcal} + \Excpt{P}{\Qcal}  \KL{Q_i}{P}  + \log \frac{2n m_i}{\delta}}{2(m_i - 1)}}  \nonumber\\
+\sqrt{\frac{ \KL{\Qcal}{\Pcal} + \log \frac{2n}{\delta}}{2(n-1)}},\nonumber
	\end{gather}
where $Q_i  \triangleq Q(S_i,P)$.
\end{Theorem}

Notice that the transfer error \eqref{def:TransferRisk} is bounded by the empirical multi-task error (\ref{eq:EmpricalMultiTaskRisk}) plus two complexity terms. The first is the average of the task-complexity  terms of the observed tasks. This term converges to zero in the limit of a large number of samples  in each task ($m_i \rightarrow \infty$). 
The second  is an environment-complexity term.  This term converges to zero if infinite number of tasks is observed from the task environment ($n \rightarrow \infty$).	
As in \citet{pentina2014pac}, our proof is based on two main steps.
The second step, similarly to \citet{pentina2014pac}, bounds the transfer-risk at the task-environment level (i.e, the error caused by observing only a finite number of tasks) by the average expected error in the observed tasks plus the environment-complexity term.
The first step differs from \citet{pentina2014pac}.
Instead of using a single joint bound on the average expected error, we use a single-task PAC-Bayes theorem to bound the expected error in each task separately (when learned using priors from the hyper-posterior), and then use a union bound argument.
By doing so our bound takes into account the specific number of samples in each observed task (instead of their harmonic mean). 
Therefore our bound is better adjusted the observed data set.

Our proof technique can utilize different single-task bounds in each of the two steps.
In section \ref{sect:MetaBoundProof} we use McAllester's bound (Theorem \ref{thm:OriginalPacBayes}), which is tighter than the lemma used in \citet{pentina2014pac}.
Therefore, the complexity terms are in the form of $\sqrt{\frac{1}{m}\KL{Q}{P}}$ instead of $\frac{1}{\sqrt{m}}\KL{Q}{P}$ as in \citet{pentina2014pac}. This means the bound is tighter 
\footnote{E.g., \citet{seldin2012pac} Theorems 5 and 6}.
In section \ref{sec:AlternativeBounds} we demonstrate how our technique can use other, possibly tighter, single-task bounds.
In Section \ref{sect:Results} we will empirically evaluate the different bounds as meta-learning objectives and show that the improved tightness is critical for performance.

\section{Meta-Learning Algorithm}

As in the single-task case, the bound  of Theorem \ref{thm:MetaBound} can be evaluated from the training data and so can serve as a minimization objective for a  principled meta-learning algorithm. Since the bound holds uniformly for all $\Qcal$, it is ensured to hold also for the minimizer of the objective $\Qcal^*$. 
{Provided that the bound is tight enough, the algorithm will approximately minimize the transfer-risk itself, avoiding overfitting to the observed tasks.}
In this section we will derive a practical learning procedure that can applied to a large family of differentiable models, including deep neural networks.

\subsection{Hyper-Posterior Model}

In this section we choose a specific form of hyper-posterior distribution $\Qcal$ which enables practical implementation.
Given  a parametric family of priors  $\braces{P_{\thetatild} : \thetatild \in \mathbb{R}^{N_P}}$, $N_P \in \mathbb{N}$,  the space of hyper-posteriors consists of all distributions over $\mathbb{R}^{N_P}$.
We will limit our search to a family of isotropic Gaussian distributions defined by 
$\Qcal_{\theta} \triangleq  \normal{\theta, \kappa_\Qcal^2 I_{N_P \times N_P}}, $ 
where $\kappa_\Qcal > 0$ is a predefined constant.
Notice that $\Qcal$ appears in the bound (\ref{eq:MetaBound}) in two forms (i)  divergence from the hyper-prior $\KL{\Qcal}{\Pcal}$ and (ii)  expectations over $P \sim \Qcal$.


By setting the hyper-prior as zero-mean isotropic Gaussian, $\Pcal = \normal{0, \kappa_\Pcal^2 I_{N_P \times N_P}}$, where $\kappa_\Pcal > 0$  is another constant, we get a simple form for the KL-divergence term, 
$\KL{\Qcal_{\theta}}{\Pcal} = \frac{\norm{\theta}_2^2 + \kappa_\Qcal^2}{2 \kappa_\Pcal^2}  + \log \frac{\kappa_\Pcal}{\kappa_\Qcal}  -\frac{1}{2} .$
Note that the  hyper-prior acts as a regularization term which prefers solutions with small $L_2$ norm.

The expectations can be approximated by averaging several Monte-Carlo samples of $P$.
Notice that  sampling from $\Qcal_{\theta}$ means adding Gaussian noise to the parameters $\theta$ during training, 
$\thetatild = \theta +  \varepsilon_P,  \varepsilon_P \sim \normal{0, \kappa_\Qcal^2 I_{N_P \times N_P}}$. 
This means the learned parameters must be robust to perturbations, which encourages selecting solutions which are less prone to over-fitting. 

\subsection{Joint Optimization}

The term appearing on the RHS of the meta-learning bound in (\ref{eq:MetaBound})  can be compactly written as 
\begin{equation} \label{eq:TotalObj}
J(\theta) \triangleq \frac{1}{n}\sum_{i=1}^{n} J_i(\theta) + \Upsilon(\theta) ,
\end{equation}	
where we defined, 
\begin{align} \label{eq:TaskObj}
&J_i(\theta)  \triangleq  \Excpt{\thetatild}{\Qcal_{\theta}}  \erhati{Q_i(S_i,P_\thetatild), S_i}{i}   \\
 +&\sqrt{\frac{\KL{\Qcal_{\theta}}{\Pcal} +\Excpt{\thetatild}{\Qcal_{\theta}}  \KL{Q(S_i,P_\thetatild)}{P_\thetatild}+ \log \frac{2n m_i}{\delta} }{2(m_i-1)}}, \nonumber\\
 & \Upsilon(\theta) \triangleq \sqrt{\frac{\KL{\Qcal_{\theta}}{\Pcal} + \log \frac{2n}{\delta}}{2(n-1)}}.
\end{align}	
Theorem \ref{thm:MetaBound} allows us to choose \textit{any} procedure $Q(S_i,P): \mathcal{Z}^{m_i} \times \mathcal{M} \rightarrow  \mathcal{M}$ as a base learner. 
We will use a procedure which minimizes $J_i(\theta)$ due to the following advantages: \textit{(i)} It minimizes a bound on the expected error of the observed task \footnote{See section \ref{sect:MetaBoundProof} in the supplementary material.}.
\textit{(ii)} It uses the prior knowledge gained from the prior $P$ to get a tighter bound and a better learning objective. \textit{(iii)} As will be shown next, formulating the single task learning as an optimization problem enables joint learning of the shared prior and the task posteriors.

To formulate the single-task learning as an optimization problem, we choose a parametric form for the posterior of each task $Q_{\phi_i}, \phi_i \in \mathbb{R}^{N_Q}$ (see section \ref{sect:DistributionModel} for an explicit example).
The base-learning algorithm can be formulated as $\phi_i^* = \argmin_{\phi_i} J_i(\theta, \phi_i)$, where we abuse notation by denoting the term $J_i(\theta)$ evaluated with posterior parameters $\phi_i$ as $J_i(\theta, \phi_i)$.
The meta-learning problem of minimizing $J(\theta)$ over $\theta$ can now be written more explicitly, 
%
\begin{equation} \label{eq:meta_learn_problem}
\min_{\theta, \phi_1,...,\phi_n} \left\{\frac{1}{n}\sum_{i=1}^{n}  J_i(\theta, \phi_i) + \Upsilon(\theta)\right\}.
\end{equation}
The optimization process is illustrated in Figure \ref{fig:algorithm}.

\begin{figure}[t]
	\begin{center}
		\includegraphics[scale=0.2]{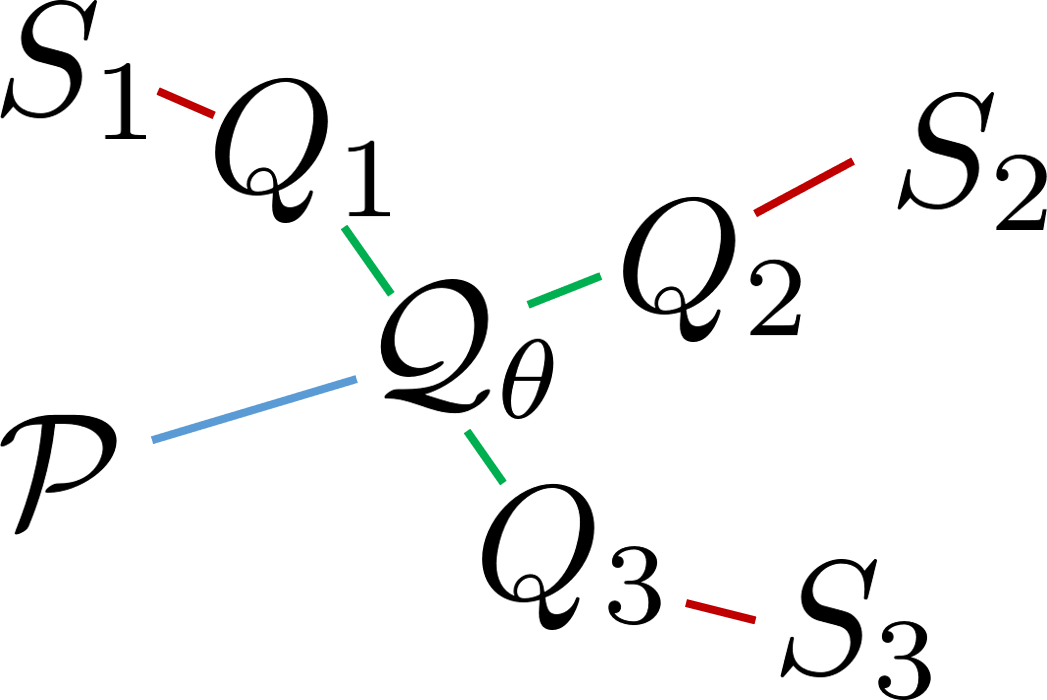}
	\end{center}
	\caption{\textbf{Joint optimization illustration.} The posterior of task $i$,  $Q_i = Q_{\phi_i}$ is influenced by both the dataset $S_i$ (through the empirical error term) and by the hyper-posterior $\Qcal_{\theta}$ (through the task-complexity term).
    The hyper-posterior $\Qcal_{\theta}$ is influenced by  both the posteriors and by the hyper-prior $\Pcal$ (through the environment-complexity term). 
	}\label{fig:algorithm}
\end{figure}


\subsection{Distributions Model} \label{sect:DistributionModel}
In this section we make the meta-learning optimization problem (\ref{eq:meta_learn_problem}) more explicit by defining a model for the posterior and prior distributions. First, we define the hypothesis class $\mathcal{H}$ as  a family of functions parameterized by a\textit{ weight vector} $\braces{h_w:w \in \mathbb{R}^d}$.
Given this parameterization, the posterior and prior are distributions over $\mathbb{R}^d$.

We will present an algorithm for any  differentiable model \footnote{The only assumption on $\braces{h_w:w \in \mathbb{R}^d}$  is that the loss function $\ell(h_w,z)$ is differentiable w.r.t $w$.}, 
but our aim is to use neural network (NN) architectures.
In fact, we will use Stochastic NNs \citep{graves2011practical, blundell2015weight}
since in our setting the weights are random and we are optimizing their posterior distribution.
The techniques presented next will be mostly based on \citet{blundell2015weight}.
Next we define the posteriors $Q_{\phi_i},~i=1,...,n,$ and the prior $P_{\theta}$ as factorized Gaussian distributions\footnote{This choice makes optimization easier, but in principle we can use other distributions as long as the density function is differentiable w.r.t.~the parameters.},
\begin{equation} \label{eq:Q_distrib}
P_{\theta}(w) = \prod_{k=1}^{d} \normal{w_k; \mu_{P,k}, \sigma_{P,k}^2} 
\end{equation}
\begin{equation}
Q_{\phi_i}(w) = \prod_{k=1}^{d} \normal{w_k; \mu_{i,k}, \sigma_{i,k}^2}
\end{equation}
where for each task, the posterior parameters vector  $\phi_i=(\mu_{i}, \rho_{i})\in \mathbb{R}^{2d}$ is composed of the means and log-variances of each weight , $\mu_{i,k}$ and $\rho_{i,k} = \log {\sigma^2_{P,k}}, k=1,...,d$.\footnote{Note that we use $\rho = \log \sigma^2$  as a parameter in order to keep the parameters unconstrained (while  $\sigma^2 = \exp(\rho)$ is guaranteed to be strictly positive).}
The shared prior vector $\theta=(\mu_{P}, \rho_{P}) \in \mathbb{R}^{2d}$  has a similar structure.
Since we aim to use deep models where $d$ could be in the order of millions, distributions with more parameters might be impractical. 

Since $Q_{\phi_i}$ and $P_{\theta}$ are factorized Gaussian distributions the KL-divergence, $\KL{Q_{\phi_i}}{P_{\theta}}$, takes a simple analytic form, 
\begin{equation} \label{eq:KLDTerm}
\frac{1}{2}  \sum_{k=1}^{d}   \braces{  \log \frac{\sigma^2_{P,k}}{\sigma^2_{i,k}}  + 
	\frac{\sigma^2_{i,k} + \pth{\mu_{i,k} - \mu_{P,k}}^2}{\sigma^2_{P,k}} - 1}.
\end{equation}


\subsection{Optimization Technique}
As an underlying optimization method, we will use stochastic gradient descent (SGD).
In each iteration, the algorithm takes a parameter step in a direction of an estimated negative gradient. As is well known, lower variance facilitates convergence and its speed.
Recall that each single-task bound is composed of an empirical error term and a complexity term (\ref{eq:TaskObj}). 
The complexity term is a simple function of $\KL{Q_{\phi_i}}{P_{\theta}}$  (\ref{eq:KLDTerm}), which can easily be differentiated analytically. However, evaluating the gradient of the empirical error term  is more challenging.

Recall the definition of the empirical error, $\erhat[Q_{\phi_i},S_i]  = \mathbb{E}_{w \sim Q_{\phi_i}} (1/m_i)\sum_{j=1}^{m_i}\loss{h_w, z_{i,j}}$. 
This term poses two major challenges. \textit{(i)} The data set $S_i$ could be very large making it expensive to cycle over all the $m_i$ samples. \textit{(ii)} The term $\loss{h_w, z_j}$ might be highly non-linear in $w$, rendering the expectation intractable.	Still, we can get an unbiased and low variance estimate of the gradient.

First, instead of using all of the data for each gradient estimation we will use a randomly sampled mini-batch $S'_i\subset S_i $. 
%
Next, we require an estimate of a gradient of the form $\Grad{\phi} \Excpt{w}{Q_{\phi}}  f(w)$ which is a common problem in machine learning. We will use the `re-parametrization trick' \citep{kingma2013auto, rezende2014stochastic} which is an efficient and low variance method
\footnote{In fact, we will use the `\textbf{local} re-parameterization trick' \citep{kingma2015variational} in which we sample a different $\varepsilon$ for each data point in the batch, which reduces the variance of the estimate. To make the computation more efficient with neural-networks, the random number generation is performed w.r.t the
	activations instead of the weights (see \citet{kingma2015variational} for more details.).} 
.
The re-parametrization trick is easily applicable in our setup since we are using Gaussian distributions. The trick is based on describing the Gaussian distribution  $w \sim Q_{\phi_i}$ (\ref{eq:Q_distrib}) as first drawing $\varepsilon \sim \normal{\bar{0}, I_{d \times d}}$ 
and then applying the deterministic function $w(\phi_i , \varepsilon) = \mu_i + \sigma_i \odot \varepsilon$
(where $\odot$  is an element-wise multiplication).
Therefore, we  get 
$
\Grad{\phi} \Excpt{w}{Q_{\phi}}  f(w) = \Grad{\phi} \Excpt{\varepsilon}{\normal{\bar{0}, I_{d \times d}}}  f(w(\phi_i , \varepsilon)).
$
The expectation can be approximated by averaging a small number of Monte-Carlo samples with reasonable accuracy. For a fixed sampled $\varepsilon$, the gradient $\Grad{\phi} f(w(\phi_i , \varepsilon))$ is easily computable with backpropagation.

In summary, the Meta-Learning by Adjusting Priors (MLAP) algorithm  is composed of two phases 	
In the first phase (Algorithm \ref{algorithm:MetaTrain}, termed ``meta-training") several observed ``training tasks" are used to learn a prior.
In the second phase (Algorithm \ref{algorithm:MetaTest}, termed ``meta-testing") the previously learned prior is used for the learning of a new task (which was unobserved in the first phase).
Note that the first phase can be used independently as a multi-task learning method. Both algorithms are described in pseudo-code in the supplementary material (section \ref{sect:code})
\footnote{Code is available at:  \url{https://github.com/ron-amit/meta-learning-adjusting-priors}.}
\footnote{For a visual illustration of the algorithm using a toy example see section \ref{sect:ToyExample} in the supplementary material.}.

\section{Experimental Demonstration} \label{sect:Results}

In this section we demonstrate the performance of our transfer method with image classification tasks solved by deep neural networks.
In image classification, the data samples, $z \triangleq (x,y)$, consist of a an image, $x$, and a label, $y$.
The hypothesis class $\braces{h_w : w \in \mathbb{R}^d}$ is the set of neural networks with a given architecture (which will be specified later).
As a loss function $\ell(h_w,z)$ we will use the cross-entropy loss.
While the theoretical framework is defined with a bounded loss, in our experiments we use an unbounded loss function in the learning objective. Still, we can have theoretical guarantees on a variation of the loss which is clipped to $[0,1]$. Furthermore, in practice the loss function is almost always smaller than one.

We conduct two experiments with two different task environments, based on augmentations of the MNIST dataset \citep{lecun1998mnist}.
In the first environment, termed \textit{permuted labels}, each task is created by a random permutation of the labels.
In the second environment,  termed \textit{permuted pixels},  each task is created by a permutation of the image pixels. 
The pixel permutations are created by  a limited number of location swaps to ensure that the tasks stay reasonably related.

In both experiments, the meta-training set is composed of tasks from the environment with $60,000$ training examples.
Following the meta-training phase, the learned prior is used to learn a new meta-test task with fewer training samples ($2,000$).
The network architecture used for the permuted-labels experiment is a small CNN with $2$ convolutional-layers, a linear hidden layer and a linear output layer.
In the permuted-pixels experiment we used a fully-connected network with 3 hidden layers and a linear output layer.
See section  \ref{sect:ImplemetDetails} for more implementation details.



{We compare the average test error of learning a new task from each environment when using the following methods.
As a baseline, we  measure the performance of learning from scratch, i.e., with no transfer from the meta-training tasks. }\textbf{Scratch-D}: deterministic (standard) learning from scratch.
\textbf{Scratch-S}: stochastic learning from scratch (using a stochastic network with no prior/complexity term).

Other methods transfer knowledge from only \textbf{one} of the meta-training tasks.
\textbf{Warm-start}: Standard learning with initial weights taken from the standard learning of a single task from the meta-training set.
\textbf{Oracle}: Same as the previous method, but some of the layers are frozen (unchanged from their initial value) depending on the experiment.
	In the permuted labels experiment  all layers besides the output are frozen. 
	In the permuted pixels we freeze all layers except the input layer.
  We refer to this method as `oracle' since the transfer technique is tailored to each task-environment, while the other methods are applied identically in any environment (and so must learn to adjust to the environment automatically).

Finally, we compare methods which transfer knowledge from \textbf{all} of the training tasks:
	  \textbf{MLAP-M}: The objective is based on Theorem \ref{thm:MetaBound} - the meta-learning bound obtained using Theorem \ref{thm:OriginalPacBayes} (McAllester's single-task bound).	
	  \textbf{MLAP-S}: The objective is based on the meta-learning bound derived from Seeger's single-task bound
      (see section \ref{sec:AlternativeBounds} in the supplementary material, eq.\refeq{eq:MetaBoundSeeger}).
	  \textbf{MLAP-PL}:  In this method we use the main theorem of  \citet{pentina2014pac} as an objective for the algorithm, instead of Theorem \ref{thm:MetaBound}.	
	  \textbf{MLAP-VB}:
		In this method the learning objective is derived from a Hierarchal Bayesian framework using variational Bayes tools
\footnote{See section \ref{sect:VariationalBayes} in the supplementary material for details.	 The explicit learning objective is in equation \refeq{eq:VariationalObjective}.}.
	  \textbf{Averaged}: Each of the training tasks is learned in a standard way to obtain a weights vector, $w_i$. 
	The learned prior is set as an isotropic Gaussian with unit variances and a mean vector which is the average of $w_i,i=1,..,n$.
	This  prior is used for meta-testing as in MLAP-S.		
	  \textbf{MAML}: The Model-Agnostic-Meta-Learning (MAML) algorithm by \citet{finn2017model}. 
   In MAML the base learner takes few gradient steps from an initial point, $\theta$, to adapt to a task.
   The meta-learner optimizes $\theta$ based on the sum of losses on the observed tasks after base-learning.
	We tested several hyper-parameters and report the best results (see details in  the supplementary material \ref{sect:ImplemetDetails}).

\begin{table}[t]	
	\caption{Comparing the average test error percentage of different learning methods on  $20$ test tasks (the $\pm$ shows the $95\%$ confidence interval) in the permuted labels and permuted pixels experiments ($200$ swaps).}
\label{table:MNIST_Results}
	\vskip 0.15in
	\begin{center}
		\begin{small}
			\begin{sc}
				\begin{tabular}{lcccr}
					\toprule
					Method & Permuted Labels & Permuted Pixels \\
					\midrule
				\textbf{Scratch-S}               & $2.27 \pm 0.06$       & $ 7.92 \pm  0.22$ \\
				\textbf{Scratch-D}             & $2.82 \pm 0.06$    	& $ 7.65 \pm 0.22$ \\
				\textbf{Warm-start} 			& $1.07 \pm 0.03$       & $7.95 \pm 0.39$ \\			
				\textbf{Oracle} 		     	& $0.69 \pm 0.04$       & $6.57 \pm 0.32$ \\
					\midrule
				\textbf{MLAP-M} 				          	& $0.908 \pm 0.04$   	& $3.4 \pm 0.18$ \\	
				\textbf{MLAP-S} 				          	& $ 0.75 \pm 0.03$      & $3.54 \pm 0.2$  \\	
				\textbf{MLAP-PL} 				        & $ 82.8 \pm 5.26$ 	    & $74.9 \pm 4.03$ \\
				\textbf{MLAP-VB} 				        & $0.85  \pm 0.03$ 	    & $3.52 \pm 0.17$ \\	
				\textbf{Averaged} 				& $2.72 \pm 0.08$ 		& $7.63 \pm 0.36$ \\	
				\textbf{MAML} 							&$1.16 \pm 0.07$     	& $3.77 \pm 0.8$ \\		
					\bottomrule
				\end{tabular}
			\end{sc}
		\end{small}
	\end{center}
	\vskip -0.1in
\end{table}

Table \ref{table:MNIST_Results} summarizes the results for the permuted labels experiment with $5$ training-tasks and the permuted pixels experiment  with 200 pixel swaps and $10$ training-tasks.
In the permuted labels experiment the best results are obtained with the ``oracle" method.
Recall that the oracle method has the ``unfair" advantage of  a ``hand-engineered" transfer technique which is based on knowledge about the problem.
In contrast, the other methods  must automatically learn the task environment by observing several tasks.

The MLAP-M and MLAP-S variants of the MLAP algorithm  improves considerably over learning from scratch and over the naive warm-start transfer. 
They even improve over the ``oracle" method in the permuted pixels experiment.
The result of the MLAP-VB are close to the MLAP-M and MLAP-S variants.
{However the  MLAP-PL variant performed much worse since the complexity terms are disproportionately large compared to the empirical error terms.
This demonstrates the importance of using the tight generalization bound developed in our work as a learning objective.}
The results for the ``averaged-prior"  method are about the same as learning from scratch. Due to the high non-linearity of the problem, averaging weights was not expected to perform well.

The results of  the MLAP algorithm are slightly better than MAML.
Note that in MAML the meta-learning only infers an initial point for base-learning.
Thus there is a trade-off in choosing the number of adaptation steps.
Taking many gradient steps exploits a larger number of samples but the effect of the  initial weights diminishes.
Also, taking a large number of steps is computationally infeasible in meta-training.
Therefore MAML is especially suited for few-shot learning, which is not the case in our experiment.
In our method we infer a prior that serves both as an initial point and as a regularizer which can fix some of the weights, while allowing variation in others, depending on the amount of data.
Recent work by \citet{grant2018recasting} showed that MAML can be interpreted as an approximate empirical Bayes procedure. This interesting perspective, differs from the present contribution that is  based on generalization bounds within a non-Bayesian setting. 

{Next we investigate whether using more training tasks improves the quality of the learned prior.}
In Figure \ref{fig:error_vs_tasks_num} we plot the average test error of learning a new task based on the number of training-tasks in the different environments,
{namely the permuted labels environment, and the permuted pixels environment with $100, 200, 300$ pixel swaps. }
 { We used the MLAP-S variant of the algorithm.} The results clearly show that the more tasks are used to learn the prior, the better the performance on the new task.
For example,  in the permuted labels case, a prior that is learned based on one or two tasks leads to negative transfer, i.e, worse results than standard learning from scratch (with no transfer),  which achieves $2.27\%$ error. However after observing $3$ or more tasks, the transfered prior facilitates learning with lower expected error. 
In the permuted pixels experiment, standard learning from scratch achieves  $7.9\%$ test error. 
The number of training tasks needed for positive transfer depends on the number of pixels swapped.
A higher number of swaps means larger variation in the task environment  and more training-tasks are needed to learn a beneficial prior.





\begin{figure}
\vskip 0.2in
\begin{center}
	 \subfigure[]{ 
		\includegraphics[width=0.2\textwidth]{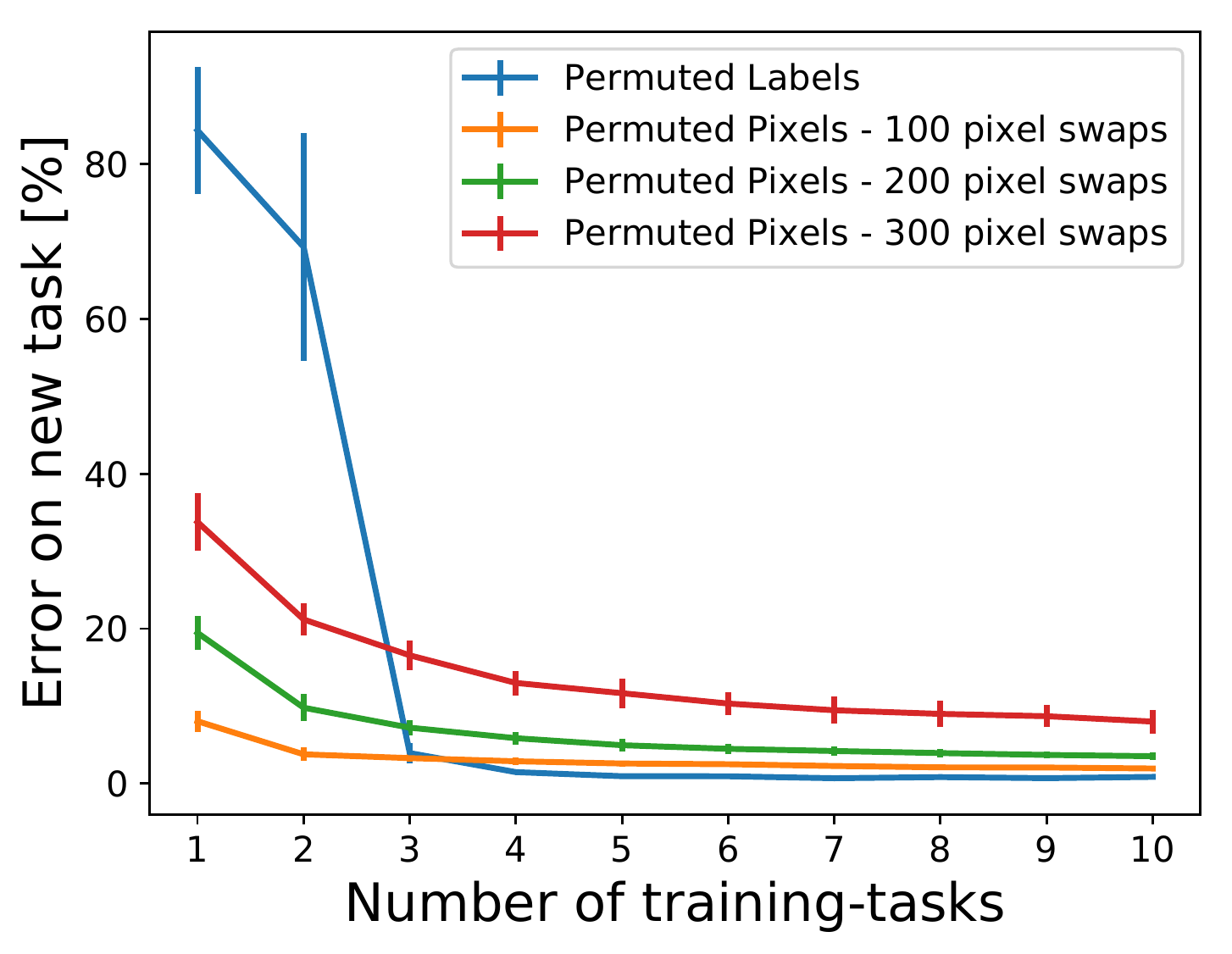}
	}
	 \subfigure[]{ 		
		\includegraphics[width = 0.2\textwidth]{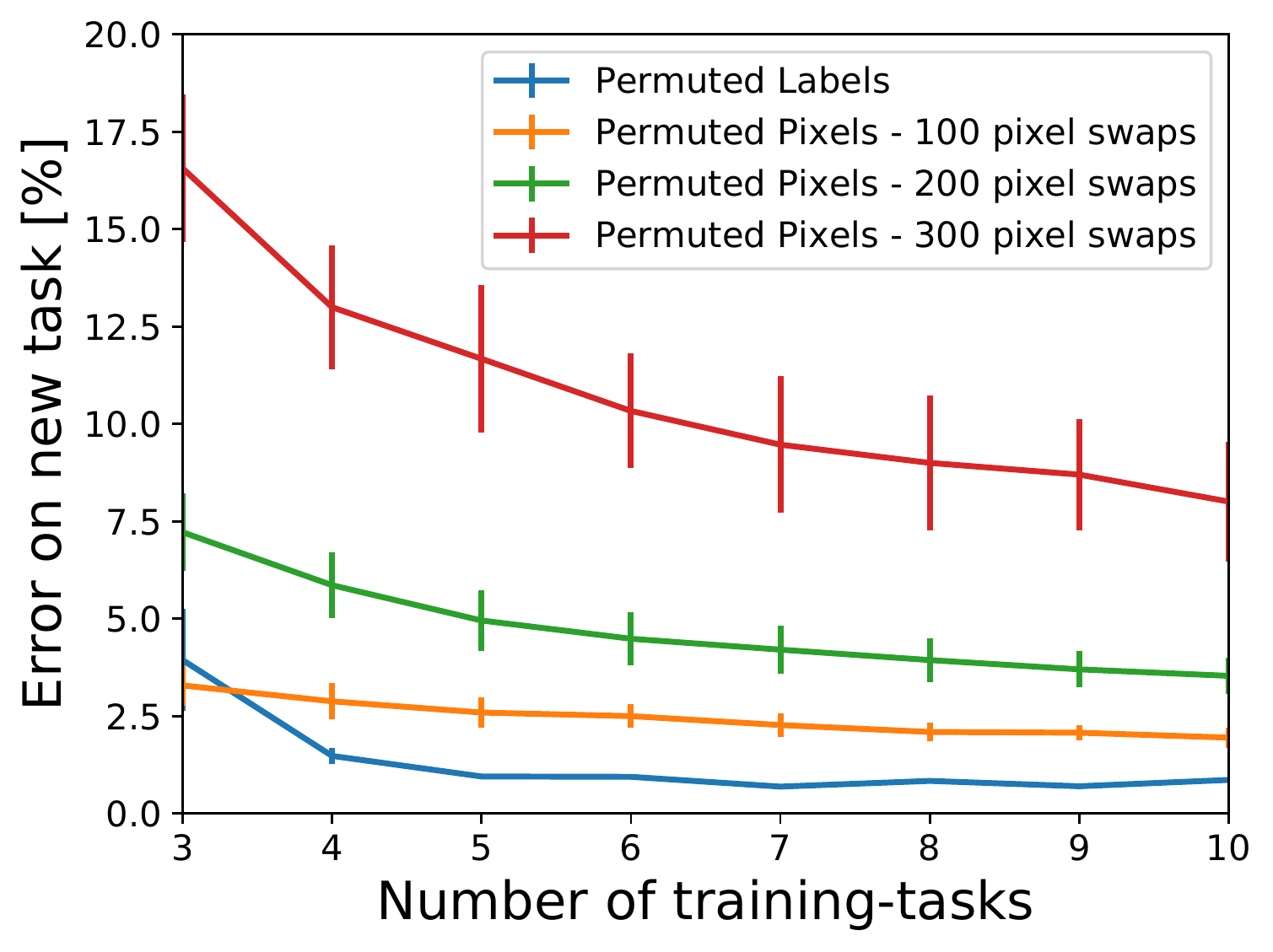}
	}
	\caption{The average test error of learning a new task for different numbers of training-tasks and for different environments (average over $20$ meta-test tasks). Figure (b) reproduces (a) starting with $3$ tasks. 
    Best viewed in color.}
  \label{fig:error_vs_tasks_num}
\end{center}
\vskip -0.2in
\end{figure}


\paragraph{Analysis of learned prior}
Qualitative examination of the learned prior affirms that it has indeed adjusted to  each task environment. 
{
In Figure \ref{fig:Log-Var-Analysis} we inspect the average log-variance parameter the learned prior assigns to the weights of each layer in the network. Higher values of this parameter indicate that the weight is more flexible to change. i.e, it is more weakly penalized for deviating form the nominal prior value.}
In the permuted-labels experiment the learned prior assigns low variance to the lower layers (fixed representation) and high variance to the output layer (which enable easy adjustment to different label permutations).
As expected, in the permuted-pixels experiment the opposite phenomenon occurs. 
The mapping from the final hidden layer to the output becomes fixed, and the mapping from the input to the final hidden layer (representation) has more flexibility to change in light of the task data.

\begin{figure}
	\vskip 0.2in
	\begin{center}
	 \subfigure[]{ 
		\includegraphics[width = 0.47\linewidth]{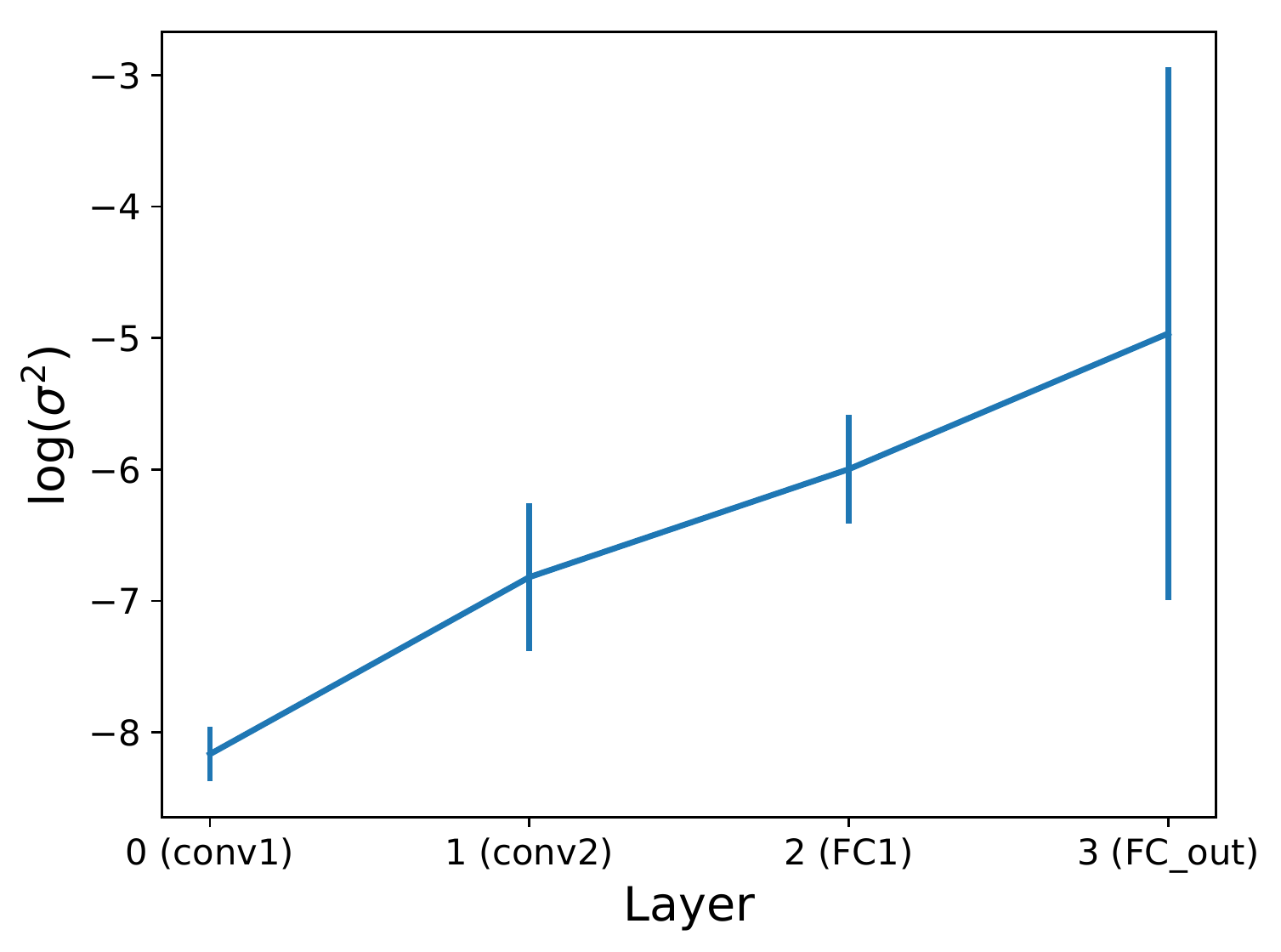}
	}
	 \subfigure[]{ 
		\includegraphics[width = 0.47\linewidth]{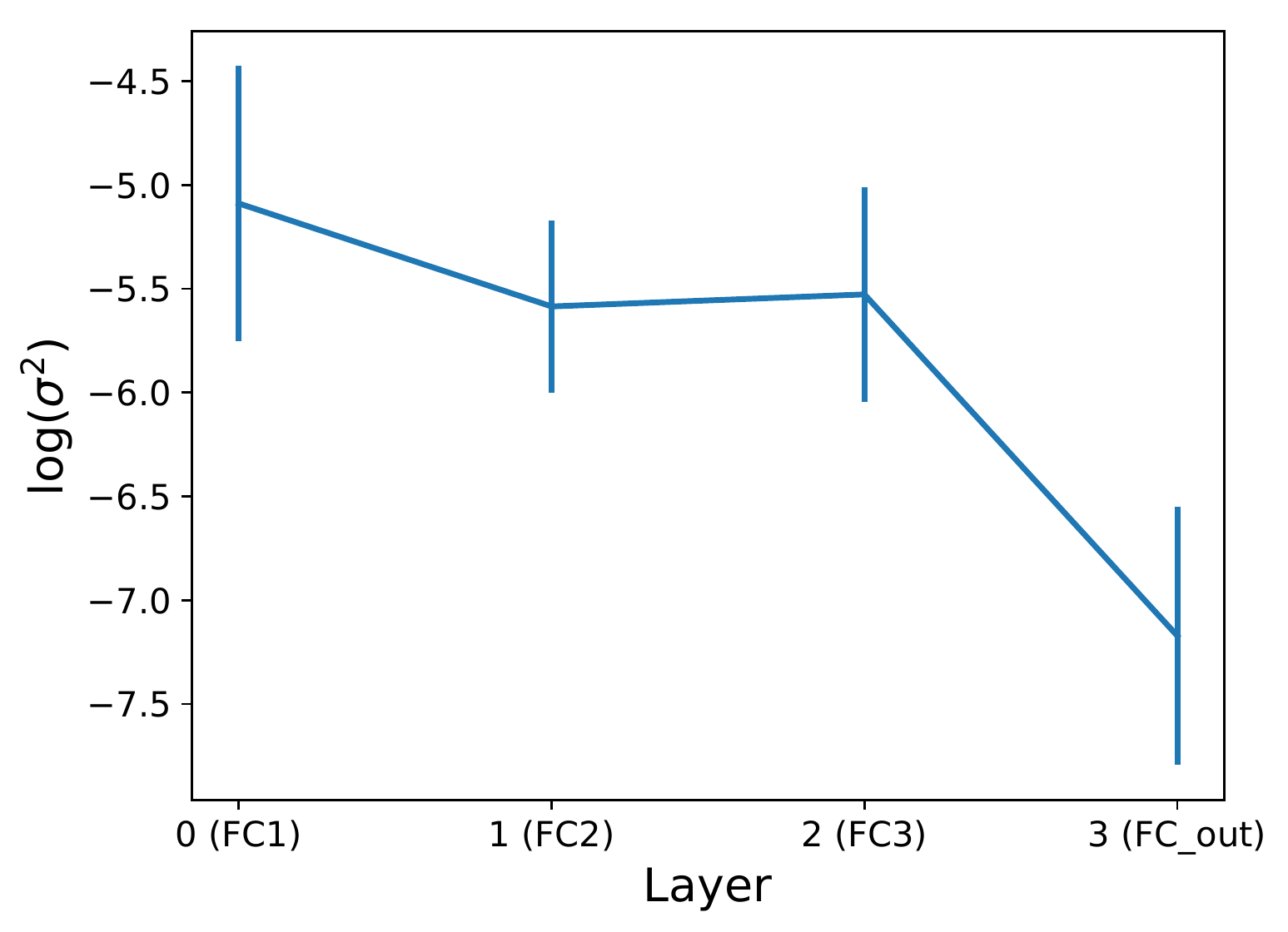}
}
	\caption{Log weight uncertainty ($\log \pth{\sigma^2}$) in each layer of the learned prior (average $\pm $ STD). Higher value means higher variance/uncertainty.
	 	(a) Permuted labels experiment, (b) Permuted pixels experiment ($200$ swaps).} 
	\label{fig:Log-Var-Analysis}
	\end{center}
	\vskip -0.2in
\end{figure}



\section{Discussion and Future Work}
We have presented a framework for meta-learning, motivated by extended PAC-Bayes generalization bounds, and implemented through the adjustment of a learned prior, based on tasks encountered so far. The framework bears conceptual similarity to the empirical Bayes method while not being Bayesian, and is implemented at the level of tasks rather than samples (see Section \ref{sect:VariationalBayes} in the supplementary material for details about a Bayesian perspective). Combining the flexibility of the approach, with the rich representational structure of deep neural networks, and learning through gradient based methods leads to an efficient procedure for meta-learning, as motivated theoretically and demonstrated empirically.  
While our experimental results are preliminary, we believe that our work attests to the utility of using rigorous performance bounds to derive learning algorithms, and demonstrates that tighter bounds indeed lead to improved performance.

There are several open issues to consider. First, the current version learns to solve all available tasks in parallel, while a more useful procedure should be sequential in nature. This can be easily incorporated into our framework by updating the prior following each novel task.  Second,  our method requires training stochastic models which is challenging due to the the high-variance gradients. We would like to develop new methods within our framework which have more stable convergence and are easier to apply in larger scale problems. Third, there is much current effort in applying meta-learning ideas to reinforcement learning, for example, \cite{teh2017distral} presents a heuristically motivated framework that is conceptually similar to ours. An interesting challenge would be to extend our techniques to derive meta-learning algorithms for reinforcement learning based on performance bounds. 

	\subsubsection*{Acknowledgments}
    
We thank Asaf Cassel, Guy Tennenholtz, Baruch Epstein, Daniel Soudry, Elad Hoffer and Tom Zahavy for helpful discussions of this work, and the anonymous reviewers for their helpful comment.
We gratefully acknowledge the support of NVIDIA Corporation with the donation of the Titan Xp GPU used for this research. The work was partially supported by the Ollendorff Center of the Viterbi Faculty of Electrical Engineering at the Technion.


\begin{thebibliography}{39}
\providecommand{\natexlab}[1]{#1}
\providecommand{\url}[1]{\texttt{#1}}
\expandafter\ifx\csname urlstyle\endcsname\relax
  \providecommand{\doi}[1]{doi: #1}\else
  \providecommand{\doi}{doi: \begingroup \urlstyle{rm}\Url}\fi

\bibitem[Alquier et~al.(2017)Alquier, Mai, and Pontil]{alquier2017regret}
Alquier, P., Mai, T.~T., and Pontil, M.
\newblock {Regret Bounds for Lifelong Learning}.
\newblock In \emph{International Conference on Artificial Intelligence and
  Statistics (AISTATS)}, pp.\  261--269, 2017.

\bibitem[Andrychowicz et~al.(2016)Andrychowicz, Denil, Gomez, Hoffman, Pfau,
  Schaul, and de~Freitas]{andrychowicz2016learning}
Andrychowicz, M., Denil, M., Gomez, S., Hoffman, M.~W., Pfau, D., Schaul, T.,
  and de~Freitas, N.
\newblock Learning to learn by gradient descent by gradient descent.
\newblock In \emph{Advances in Neural Information Processing Systems (NIPS)},
  pp.\  3981--3989, 2016.

\bibitem[Audibert(2010)]{audibert2010pac}
Audibert, J.-Y.
\newblock \emph{{PAC-Bayesian} aggregation and multi-armed bandits}.
\newblock PhD thesis, Universit{\'e} Paris-Est, 2010.

\bibitem[Baxter(2000)]{baxter2000model}
Baxter, J.
\newblock A model of inductive bias learning.
\newblock \emph{J. Artif. Intell. Res.(JAIR)}, 12\penalty0 (149-198):\penalty0
  3, 2000.

\bibitem[Ben-David et~al.(2010)Ben-David, Blitzer, Crammer, Kulesza, Pereira,
  and Vaughan]{ben2010theory}
Ben-David, S., Blitzer, J., Crammer, K., Kulesza, A., Pereira, F., and Vaughan,
  J.~W.
\newblock A theory of learning from different domains.
\newblock \emph{Machine learning}, 79\penalty0 (1-2):\penalty0 151--175, 2010.

\bibitem[Blundell et~al.(2015)Blundell, Cornebise, Kavukcuoglu, and
  Wierstra]{blundell2015weight}
Blundell, C., Cornebise, J., Kavukcuoglu, K., and Wierstra, D.
\newblock Weight uncertainty in neural network.
\newblock In \emph{International Conference on Machine Learning (ICML)}, pp.\
  1613--1622, 2015.

\bibitem[Caruana(1997)]{caruana1998multitask}
Caruana, R.
\newblock Multitask learning.
\newblock \emph{Machine Learning}, 28\penalty0 (1):\penalty0 41--75, 1997.

\bibitem[Catoni(2007)]{catoni2007pac}
Catoni, O.
\newblock {PAC-Bayesian} supervised classification.
\newblock \emph{Lecture Notes-Monograph Series. IMS}, 2007.

\bibitem[Devroye et~al.(1996)Devroye, Gyo{\"{o}}rfi, and Lugosi]{DevGyoLug96}
Devroye, L., Gyo{\"{o}}rfi, L., and Lugosi, G.
\newblock \emph{{A Probabilistic Theory of Pattern Recognition}}.
\newblock Springer, 1996.

\bibitem[Dziugaite \& Roy(2017)Dziugaite and Roy]{dziugaite2017computing}
Dziugaite, G.~K. and Roy, D.~M.
\newblock Computing nonvacuous generalization bounds for deep (stochastic)
  neural networks with many more parameters than training data.
\newblock In \emph{Conference on Uncertainty in Artificial Intelligence,
  (UAI)}, 2017.

\bibitem[Edwards \& Storkey(2016)Edwards and Storkey]{edwards2016towards}
Edwards, H. and Storkey, A.
\newblock Towards a neural statistician.
\newblock In \emph{International Conference on Learning Representations
  (ICLR)}, 2016.

\bibitem[Finn et~al.(2017)Finn, Abbeel, and Levine]{finn2017model}
Finn, C., Abbeel, P., and Levine, S.
\newblock Model-agnostic meta-learning for fast adaptation of deep networks.
\newblock In \emph{International Conference on Machine Learning (ICML)}, pp.\
  1126--1135, 2017.

\bibitem[Galanti et~al.(2016)Galanti, Wolf, and Hazan]{galanti2016theoretical}
Galanti, T., Wolf, L., and Hazan, T.
\newblock A theoretical framework for deep transfer learning.
\newblock \emph{Information and Inference: A Journal of the IMA}, 5\penalty0
  (2):\penalty0 159--209, 2016.

\bibitem[Germain et~al.(2016)Germain, Bach, Lacoste, and
  Lacoste-Julien]{germain2016pac}
Germain, P., Bach, F., Lacoste, A., and Lacoste-Julien, S.
\newblock {PAC-Bayesian} theory meets {{Bayesian}} inference.
\newblock In \emph{Advances In Neural Information Processing Systems (NIPS)},
  pp.\  1876--1884, 2016.

\bibitem[Grant et~al.(2018)Grant, Finn, Levine, Darrell, and
  Griffiths]{grant2018recasting}
Grant, E., Finn, C., Levine, S., Darrell, T., and Griffiths, T.
\newblock Recasting gradient-based meta-learning as hierarchical {Bayes}.
\newblock In \emph{International Conference on Learning Representations
  (ICLR)}, 2018.

\bibitem[Graves(2011)]{graves2011practical}
Graves, A.
\newblock Practical variational inference for neural networks.
\newblock In \emph{Advances in Neural Information Processing Systems (NIPS)},
  pp.\  2348--2356, 2011.

\bibitem[Kingma \& Welling(2013)Kingma and Welling]{kingma2013auto}
Kingma, D.~P. and Welling, M.
\newblock Auto-encoding variational {Bayes}.
\newblock In \emph{International Conference on Learning Representations
  (ICLR)}, 2013.

\bibitem[Kingma et~al.(2015)Kingma, Salimans, and
  Welling]{kingma2015variational}
Kingma, D.~P., Salimans, T., and Welling, M.
\newblock Variational dropout and the local reparameterization trick.
\newblock In \emph{Advances in Neural Information Processing Systems (NIPS)},
  pp.\  2575--2583, 2015.

\bibitem[Kirkpatrick et~al.(2017)Kirkpatrick, Pascanu, Rabinowitz, Veness,
  Desjardins, Rusu, Milan, Quan, Ramalho, Grabska-Barwinska,
  et~al.]{kirkpatrick2017overcoming}
Kirkpatrick, J., Pascanu, R., Rabinowitz, N., Veness, J., Desjardins, G., Rusu,
  A.~A., Milan, K., Quan, J., Ramalho, T., Grabska-Barwinska, A., et~al.
\newblock Overcoming catastrophic forgetting in neural networks.
\newblock \emph{Proceedings of the National Academy of Sciences (PNAS)}, pp.\
  201611835, 2017.

\bibitem[LeCun(1998)]{lecun1998mnist}
LeCun, Y.
\newblock The mnist database of handwritten digits.
\newblock \emph{http://yann. lecun. com/exdb/mnist/}, 1998.

\bibitem[LeCun et~al.(2015)LeCun, Bengio, and Hinton]{lecun2015deep}
LeCun, Y., Bengio, Y., and Hinton, G.
\newblock Deep learning.
\newblock \emph{Nature}, 521\penalty0 (7553):\penalty0 436, 2015.

\bibitem[Lever et~al.(2013)Lever, Laviolette, and
  Shawe-Taylor]{lever2013tighter}
Lever, G., Laviolette, F., and Shawe-Taylor, J.
\newblock Tighter {PAC-Bayes} bounds through distribution-dependent priors.
\newblock \emph{Theoretical Computer Science}, 473:\penalty0 4--28, 2013.

\bibitem[Maurer(2005)]{maurer2005algorithmic}
Maurer, A.
\newblock Algorithmic stability and meta-learning.
\newblock \emph{Journal of Machine Learning Research (JMLR)}, 6:\penalty0
  967--994, 2005.

\bibitem[Maurer(2009)]{maurer2009transfer}
Maurer, A.
\newblock Transfer bounds for linear feature learning.
\newblock \emph{Machine learning}, 75\penalty0 (3):\penalty0 327--350, 2009.

\bibitem[Maurer et~al.(2016)Maurer, Pontil, and
  Romera-Paredes]{maurer2016benefit}
Maurer, A., Pontil, M., and Romera-Paredes, B.
\newblock The benefit of multitask representation learning.
\newblock \emph{Journal of Machine Learning Research (JMLR)}, 17\penalty0
  (81):\penalty0 1--32, 2016.

\bibitem[McAllester(1999)]{mcallester1999pac}
McAllester, D.~A.
\newblock {PAC-Bayesian} model averaging.
\newblock In \emph{Conference on Computational Learning Theory (COLT)}, pp.\
  164--170, 1999.

\bibitem[McNamara \& Balcan(2017)McNamara and Balcan]{mcnamara2017risk}
McNamara, D. and Balcan, M.-F.
\newblock Risk bounds for transferring representations with and without
  fine-tuning.
\newblock In \emph{International Conference on Machine Learning (ICML)}, pp.\
  2373--2381, 2017.

\bibitem[Pentina \& Lampert(2014)Pentina and Lampert]{pentina2014pac}
Pentina, A. and Lampert, C.~H.
\newblock A {PAC-Bayesian} bound for lifelong learning.
\newblock In \emph{International Conference on Machine (ICML)}, pp.\  991--999,
  2014.

\bibitem[Pentina \& Lampert(2015)Pentina and Lampert]{pentina2015lifelong}
Pentina, A. and Lampert, C.~H.
\newblock Lifelong learning with non-iid tasks.
\newblock In \emph{Advances in Neural Information Processing Systems (NIPS)},
  pp.\  1540--1548, 2015.

\bibitem[Rezende et~al.(2014)Rezende, Mohamed, and
  Wierstra]{rezende2014stochastic}
Rezende, D.~J., Mohamed, S., and Wierstra, D.
\newblock Stochastic backpropagation and approximate inference in deep
  generative models.
\newblock In \emph{International Conference on Machine Learning (ICML)}, pp.\
  1278--1286, 2014.

\bibitem[Ruvolo \& Eaton(2013)Ruvolo and Eaton]{ruvolo2013ella}
Ruvolo, P. and Eaton, E.
\newblock {ELLA}: An efficient lifelong learning algorithm.
\newblock In \emph{International Conference on Machine Learning (ICML)}, pp.\
  507--515, 2013.

\bibitem[Seeger(2002)]{seeger2002pac}
Seeger, M.
\newblock {PAC-Bayesian} generalisation error bounds for gaussian process
  classification.
\newblock \emph{Journal of Machine Learning Research (JMLR)}, 3\penalty0
  (Oct):\penalty0 233--269, 2002.

\bibitem[Seldin et~al.(2012)Seldin, Laviolette, Cesa-Bianchi, Shawe-Taylor, and
  Auer]{seldin2012pac}
Seldin, Y., Laviolette, F., Cesa-Bianchi, N., Shawe-Taylor, J., and Auer, P.
\newblock {PAC-Bayesian} inequalities for martingales.
\newblock \emph{IEEE Transactions on Information Theory}, 58\penalty0
  (12):\penalty0 7086--7093, 2012.

\bibitem[Teh et~al.(2017)Teh, Bapst, Czarnecki, Quan, Kirkpatrick, Hadsell,
  Heess, and Pascanu]{teh2017distral}
Teh, Y., Bapst, V., Czarnecki, W.~M., Quan, J., Kirkpatrick, J., Hadsell, R.,
  Heess, N., and Pascanu, R.
\newblock Distral: Robust multitask reinforcement learning.
\newblock In \emph{Advances in Neural Information Processing Systems (NIPS)},
  pp.\  4499--4509, 2017.

\bibitem[Thrun(1996)]{thrun1996learning}
Thrun, S.
\newblock Is learning the n-th thing any easier than learning the first?
\newblock In \emph{Advances in neural information processing systems (NIPS)},
  pp.\  640--646, 1996.

\bibitem[Thrun \& Pratt(1997)Thrun and Pratt]{Thrun-1997-14508}
Thrun, S. and Pratt, L.
\newblock \emph{Learning To Learn}.
\newblock Kluwer Academic Publishers, November 1997.

\bibitem[Vilalta \& Drissi(2002)Vilalta and Drissi]{vilalta2002perspective}
Vilalta, R. and Drissi, Y.
\newblock A perspective view and survey of meta-learning.
\newblock \emph{Artificial Intelligence Review}, 18\penalty0 (2):\penalty0
  77--95, 2002.

\bibitem[Yin \& Pan(2017)Yin and Pan]{yin2017knowledge}
Yin, H. and Pan, S.~J.
\newblock Knowledge transfer for deep reinforcement learning with hierarchical
  experience replay.
\newblock In \emph{AAAI Conference on Artificial Intelligence}, pp.\
  1640--1646, 2017.

\bibitem[Yosinski et~al.(2014)Yosinski, Clune, Bengio, and
  Lipson]{yosinski2014transferable}
Yosinski, J., Clune, J., Bengio, Y., and Lipson, H.
\newblock How transferable are features in deep neural networks?
\newblock In \emph{Advances in neural information processing systems (NIPS)},
  pp.\  3320--3328, 2014.

\end{thebibliography}


\begin{thebibliography}{14}
\providecommand{\natexlab}[1]{#1}
\providecommand{\url}[1]{\texttt{#1}}
\expandafter\ifx\csname urlstyle\endcsname\relax
  \providecommand{\doi}[1]{doi: #1}\else
  \providecommand{\doi}{doi: \begingroup \urlstyle{rm}\Url}\fi

\bibitem[Alquier \& Guedj(2018)Alquier and Guedj]{alquier2018simpler}
Alquier, P. and Guedj, B.
\newblock Simpler {PAC-Bayesian} bounds for hostile data.
\newblock \emph{Machine Learning}, 107\penalty0 (5):\penalty0 887--902, 2018.

\bibitem[Blei et~al.(2003)Blei, Ng, and Jordan]{blei2003latent}
Blei, D.~M., Ng, A.~Y., and Jordan, M.~I.
\newblock Latent {Dirichlet} allocation.
\newblock \emph{Journal of Machine Learning Research (JMLR)}, 3\penalty0
  (Jan):\penalty0 993--1022, 2003.

\bibitem[Clevert et~al.(2016)Clevert, Unterthiner, and
  Hochreiter]{clevert2015fast}
Clevert, D.-A., Unterthiner, T., and Hochreiter, S.
\newblock Fast and accurate deep network learning by exponential linear units
  {(ELUs)}.
\newblock In \emph{International Conference on Learning Representations
  (ICLR)}, 2016.

\bibitem[Edwards \& Storkey(2016)Edwards and Storkey]{edwards2016towards}
Edwards, H. and Storkey, A.
\newblock Towards a neural statistician.
\newblock In \emph{International Conference on Learning Representations
  (ICLR)}, 2016.

\bibitem[Germain et~al.(2016)Germain, Bach, Lacoste, and
  Lacoste-Julien]{germain2016pac}
Germain, P., Bach, F., Lacoste, A., and Lacoste-Julien, S.
\newblock {PAC-Bayesian} theory meets {{Bayesian}} inference.
\newblock In \emph{Advances In Neural Information Processing Systems (NIPS)},
  pp.\  1876--1884, 2016.

\bibitem[Glorot \& Bengio(2010)Glorot and Bengio]{glorot2010understanding}
Glorot, X. and Bengio, Y.
\newblock Understanding the difficulty of training deep feedforward neural
  networks.
\newblock In \emph{International Conference on Artificial Intelligence and
  Statistics (AISTATS)}, pp.\  249--256, 2010.

\bibitem[Kingma \& Ba(2015)Kingma and Ba]{kingma2014adam}
Kingma, D. and Ba, J.
\newblock Adam: A method for stochastic optimization.
\newblock In \emph{International Conference on Learning Representations
  (ICLR)}, 2015.

\bibitem[Maurer(2004)]{maurer2004note}
Maurer, A.
\newblock A note on the {PAC Bayesian} theorem.
\newblock \emph{arXiv preprint cs/0411099}, 2004.

\bibitem[McAllester(1999)]{mcallester1999pac}
McAllester, D.~A.
\newblock {PAC-Bayesian} model averaging.
\newblock In \emph{Conference on Computational Learning Theory (COLT)}, pp.\
  164--170, 1999.

\bibitem[Neyshabur et~al.(2018)Neyshabur, Bhojanapalli, and
  Srebro]{neyshabur2017pac}
Neyshabur, B., Bhojanapalli, S., and Srebro, N.
\newblock A {PAC-Bayesian} approach to spectrally-normalized margin bounds for
  neural networks.
\newblock In \emph{International Conference on Learning Representations
  (ICLR)}, 2018.

\bibitem[Seeger(2002)]{seeger2002pac}
Seeger, M.
\newblock {PAC-Bayesian} generalisation error bounds for gaussian process
  classification.
\newblock \emph{Journal of Machine Learning Research (JMLR)}, 3\penalty0
  (Oct):\penalty0 233--269, 2002.

\bibitem[Shalev-Shwartz \& Ben-David(2014)Shalev-Shwartz and
  Ben-David]{shalev2014understanding}
Shalev-Shwartz, S. and Ben-David, S.
\newblock \emph{Understanding machine learning: From theory to algorithms}.
\newblock Cambridge university press, 2014.

\bibitem[Tolstikhin \& Seldin(2013)Tolstikhin and Seldin]{tolstikhin2013pac}
Tolstikhin, I.~O. and Seldin, Y.
\newblock {PAC-Bayes}-empirical-{Bernstein} inequality.
\newblock In \emph{Advances in Neural Information Processing Systems (NIPS)},
  pp.\  109--117, 2013.

\bibitem[Zhang et~al.(2008)Zhang, Ghahramani, and Yang]{zhang2008flexible}
Zhang, J., Ghahramani, Z., and Yang, Y.
\newblock Flexible latent variable models for multi-task learning.
\newblock \emph{Machine Learning}, 73\penalty0 (3):\penalty0 221--242, 2008.

\end{thebibliography}
\putbib
\end{bibunit}

\begin{bibunit} 
\newpage
\onecolumn
\appendix




\section{Supplementary Material: Meta-Learning by Adjusting Priors Based on Extended PAC-Bayes Theory}


\subsection{Proof of the Meta-Learning Bound} \label{sect:MetaBoundProof}

In this section we prove Theorem \ref{thm:MetaBound}.
The proof is based on two steps, both use McAllaster's classical PAC-Bayes bound.
In the first step we use it to bound the error which is caused due to observing only a finite number of samples in each of the observed tasks.
In the second step we use it again to bound the generalization error due to observing a limited number of tasks from the environment.

We start by restating the classical PAC-Bayes bound  \citep{mcallester1999pac, shalev2014understanding} 
using general notations.
\begin{Theorem}[Classical PAC-Bayes bound, general notations] \label{thm:McAllasterGeneral}
	Let $\mathcal{X}$ be a sample space and $\mathbb{X}$ some distribution over $\mathcal{X} $, and let $\mathcal{F}$ be a hypothesis space of functions over $\mathcal{X}$.	Define a `loss function'  $g(f,X):\mathcal{F} \times \mathcal{X}  \rightarrow [0,1]$, and let $X_1^K \triangleq \left\{X_1,...,X_K\right\}$  be a sequence of $K$ independent random variables distributed according to $\mathbb{X}$.
	Let $\pi$ be some prior distribution over $\mathcal{F}$ (which must not depend on the samples $X_1,...,X_K$).
	For any $\delta \in (0,1]$, the following bound holds uniformly for all `posterior' distributions  $\rho$ over  $\mathcal{F}$ (even sample dependent),
	\begin{gather}    \label{eq:McTh}
	\mathbb{P}_{X_1^K \isEquivTo{i.i.d} \mathbb{X}}  \bigg\{ 
	\Excpt{X}{\mathbb{X}} \Excpt{f}{\rho} g(f,X) \leq 
	\frac{1}{K}\sum_{k=1}^{K} \Excpt{f}{\rho} g(f,X_k) +   
	\sqrt{\frac{1}{2(K-1)} \pth{\KL{\rho}{\pi} +	\log \frac{K}{\delta}}},  
	\forall \rho \bigg\}  \geq 1-\delta. 
	\end{gather}
\end{Theorem}

\paragraph{First step}
We use Theorem \ref{thm:McAllasterGeneral} to bound the generalization error in each of the observed tasks when learning is done by an algorithm $Q:\mathcal{Z}^{m_i} \times \mathcal{M} \rightarrow \mathcal{M}$ which uses a prior and the samples to output a distribution over hypotheses.

Let $i \in {1,...,n}$ be the index of some observed task.
We use Theorem \ref{thm:McAllasterGeneral} with the following substitutions.
The samples are $X_k \triangleq z_{i,j}$, $K \triangleq m_i$ ,  and their distribution is $\mathbb{X}  \triangleq  \Dcal_i$.
We define a `tuple hypothesis'$f=(P,h)$ where $P \in \mathcal{M}$ and $h \in \mathcal{H}$.
The `loss function' is the regular loss which uses only the $h$ element in the tuple, $g(f,X)  \triangleq  \ell(h,z)$.
We define the `prior over hypothesis',  $\pi \triangleq (\Pcal, P)$, as some distribution over $\mathcal{M} \times \mathcal{H}$  in which we first sample $P$ from $\Pcal$ and then sample  $h$ from $P$.
According to Theorem \ref{thm:McAllasterGeneral}, the `posterior over hypothesis' can be any distribution (even sample dependent), in particular, the bound will hold for the following family of distributions over $\mathcal{M} \times \mathcal{H}$, $\rho  \triangleq  (\Qcal, Q(S_i,P))$, in which we first sample $P$ from $\Qcal$ and then sample $h$ from $Q=Q(S_i,P)$
\footnote{Recall that $Q(S_i,P)$ is the posterior distribution which is the output of the learning algorithm  $Q()$ which uses the data $S_i$ and the prior $P$.}.

The KL-divergence term is
\begin{align*}
\KL{\rho}{\pi} &= \Excpt{f}{\rho} \log \frac{\rho(f)}{\pi(f)} = 
\Excpt{P}{\Qcal} \Excpt{h}{Q(S,P)} \log \frac{\Qcal(P)Q(S_i,P)(h)}{\Pcal(P) {P(h)}} \\
&=\Excpt{P}{\Qcal}  \log \frac{\Qcal(P)}{\Pcal(P)}  + \Excpt{P}{\Qcal}  \Excpt{h}{Q(S,P)}  \log \frac{Q(S_i,P)(h)}{P(h)}  \\
&=\KL{\Qcal}{\Pcal}  + \Excpt{P}{\Qcal} \KL{Q(S_i,P)}{P} 
\end{align*}


Plugging in to (\ref{eq:McTh}) we obtain that for any $\delta_i > 0$
\begin{gather}  
\mathbb{P}_{S_i \sim \Dcal_i^m} \bigg\{ 
\Excpt{z}{\Dcal_i} \Excpt{P}{\Qcal}  \Excpt{h}{Q(S_i,P)} \ell(h,z) \leq
\frac{1}{m_i}\sum_{j=1}^{m_i} \Excpt{P}{\Qcal}  \Excpt{h}{Q(S_i,P)} \ell(h,z_{i,j})  \\ \
+ \sqrt{\frac{1}{2(m_i - 1)} \pth{\KL{\Qcal}{\Pcal} + 
		\Excpt{P}{\Qcal}  \KL{Q(S_i,P)}{P} + \log \frac{m_i}{\delta_i}}},
\forall \Qcal \bigg\} 
\geq 1-\delta_i,   \nonumber
\end{gather}
for all observed tasks $i=1,..,n$.

Using the terms in section \ref{sect:SingleTaskDef}, we can write the above as, 
\begin{gather}  \label{eq:IntraTaskLevel}
\mathbb{P}_{S_i \sim \Dcal_i^m} \bigg\{ 
\Excpt{P}{\Qcal} \er[Q(S_i,P), \Dcal_i]  \leq
\Excpt{P}{\Qcal}  \erhat[Q(S_i,P), S_i]  \\ \
+  \sqrt{\frac{1}{2(m_i-1)} \pth{\KL{\Qcal}{\Pcal} + 
		\Excpt{P}{\Qcal}  \KL{Q(S_i,P)}{P} + \log \frac{m_i}{\delta_i}}},
\forall \Qcal \bigg\} 
\geq 1-\delta_i,   \nonumber
\end{gather}

\paragraph{Second step} 
Next we wish to bound the environment-level generalization (i.e, the error due to observing only a finite number of tasks from the environment).
We will use Theorem \ref{thm:McAllasterGeneral} again, with the following substitutions.
The  i.i.d.~samples are  $(\Dcal_i, m_i, S_i), i=1,...,n$  where $(\Dcal_i, m_i)$  are distributed according to the task-distribution $\tau$ and $S_i \sim \Dcal_i^{m_i}$. 
The `hypotheses' are $f \triangleq P$ and the `loss function' is $g(f,X)  \triangleq \Excpt{h}{Q(S,P)} \Excpt{z}{\Dcal}\ell(h,z)$.
Let $\pi \triangleq \Pcal$ be some distribution over $\mathcal{M}$, the bound will hold uniformly for all distributions $\rho \triangleq \Qcal$ over $\mathcal{M}$.

For any $\delta_0 > 0$, the following holds (according to Theorem \ref{thm:McAllasterGeneral}), 
\begin{gather}   
\mathbb{P}_{(\Dcal_i, m_i) \sim \tau,S_i \sim \Dcal_i^{m_i}, i=1,..,n}  \bigg\{  
\Excpt{(\Dcal, m)}{\tau}  \Excpt{S}{\Dcal^m}   \Excpt{P}{\Qcal}  \Excpt{h}{Q(S,P)} \Excpt{z}{\Dcal}\ell(h,z)  \leq\\  
\frac{1}{n}\sum_{i=1}^{n}  \Excpt{P}{\Qcal}    \Excpt{h}{Q(S_i,P)} \Excpt{z}{\Dcal_i}\ell(h,z) + \nonumber
\sqrt{\frac{1}{2(n-1)} \pth{\KL{\Qcal}{\Pcal} + \log \frac{n}{\delta_0}}},
\forall \Qcal \bigg\}  \geq 1-\delta_0.  \nonumber 
\end{gather}
Using the terms in section \ref{sect:MetaProblemDef}, we can write the above as,
\begin{gather}   \label{eq:InterTaskLevel}
\mathbb{P}_{(\Dcal_i, m_i) \sim \tau, S_i \sim  \Dcal_i^{m_i},  i=1,..,n}  \bigg\{  
\er[\Qcal,\tau] \leq 
\Excpt{P}{\Qcal} \frac{1}{n}\sum_{i=1}^{n}    \er[Q(S_i,P), \Dcal_i]  \\
+ \sqrt{\frac{1}{2(n-1)} \pth{\KL{\Qcal}{\Pcal} + \log \frac{n}{\delta_0}}}, 
\forall \Qcal \bigg\}  \geq 1-\delta_0.  \nonumber 
\end{gather}

Finally, we will bound the probability of the event which is the intersection of the events in (\ref{eq:IntraTaskLevel}) and (\ref{eq:InterTaskLevel}) by using the union bound.
For any $\delta > 0$, set $\delta_0 \triangleq \frac{\delta}{2}$ and $\delta_i \triangleq \frac{\delta}{2n}$ for $i=1,...,n$.

Using a union bound argument (Lemma \ref{Lem:Union}) we finally get, 

\begin{gather*} 
\mathbb{P}_{(\Dcal_i, m_i) \sim\tau, S_i \sim  \Dcal_i^{m_i},i=1,...,n} \bigg\{ \er[\Qcal, \tau] \leq
\frac{1}{n} \sum_{i=1}^{n} \Excpt{P}{\Qcal}  \erhati{Q_i(S_i,P),S_i}{i}  \nonumber \\
+ \frac{1}{n} \sum_{i=1}^{n}  \sqrt{\frac{1}{2(m_i-1)} \pth{\KL{\Qcal}{\Pcal} + 
		\Excpt{P}{\Qcal}  \KL{Q(S_i,P)}{P} + \log \frac{2  n m_i}{\delta}}}  \nonumber \\
+ \sqrt{\frac{1}{2(n-1)} \pth{\KL{\Qcal}{\Pcal} + \log \frac{2n}{\delta}}} 
, \forall \Qcal \bigg\} \geq 1 - \delta \nonumber.
\end{gather*}

\subsection{Meta-Learning Bound Based on Alternative Single-Task Bounds}  \label{sec:AlternativeBounds}
Many PAC-Bayesian bounds for single-task learning  have appeared in the literature.
In this section we demonstrate how our proof technique can be used with a different single-task bound to derive a possibly tighter meta-learning bound.

Consider the following single-task bound by \citep{seeger2002pac, maurer2004note}. 
{\footnote{Note that we used the slightly tighter version version  by \citet{maurer2004note} bound which  requires $K \geq 8$ .}}

\begin{Theorem}[Seeger's single-task bound]   \label{thm:SeegerBound}
	Under the same notations as Theorem \ref{thm:McAllasterGeneral}, for any $\delta \in (0,1]$ we have,
	\begin{gather}    \label{eq:SeegerTh}
	\mathbb{P}_{X_1,...,X_K \isEquivTo{i.i.d} \mathbb{X}}  \bigg\{ 
	\Excpt{X}{\mathbb{X}} \Excpt{f}{\rho} g(f,X) \leq 
	\erhat[\rho , X_1^K] + 
	2 \varepsilon + \sqrt{2 \varepsilon \erhat[\rho , X_1^K]},
	\forall \rho \bigg\}  \geq 1-\delta, \nonumber
	\end{gather}
	where we define,	
	\begin{equation*}
	\varepsilon(K,\rho,\pi,\delta) \triangleq \frac{1}{K} \pth{\KL{\rho}{\pi} + \log \frac{2\sqrt{K}}{\delta}} ,
	\end{equation*}
	and, 
	\begin{equation*}
	\erhat[\rho , X_1^K] \triangleq \frac{1}{K}\sum_{k=1}^{K} \Excpt{f}{\rho} g(f,X_k).
	\end{equation*}
\end{Theorem}
Using the above theorem we get an alternative intra-task bound to (\ref{eq:IntraTaskLevel}),
\begin{gather}   \label{SeegerIntraTask}
\mathbb{P}_{S_i \sim \Dcal_i^m} \bigg\{ 
\Excpt{P}{\Qcal} \er[Q(S_i,P), \Dcal_i]  \leq
\Excpt{P}{\Qcal}  \erhat[Q(S_i,P), S_i] \\
+ 2 \varepsilon_i + \sqrt{2 \varepsilon_i \erhat[Q(S_i,P), S_i] } 
, \forall \Qcal \bigg\} 
\geq 1-\delta_i,   \nonumber
\end{gather}
where,
\begin{equation*}
\varepsilon_i \triangleq \frac{1}{m_i} \pth{\KL{\Qcal}{\Pcal} + \Excpt{P}{\Qcal}  \KL{Q(S_i,P)}{P} + \log \frac{2\sqrt{m_i}}{\delta_i}} .
\end{equation*}

While the classical bound of Theorem \ref{thm:OriginalPacBayes} converges at a rate of $O(1/\sqrt{m})$ (as in basic VC-like bounds), the bound of Theorem \ref{thm:SeegerBound} converges faster (at a rate of  $O(1/m)$) if the empirical error $\erhat[Q]$ is negligibly small (compared to $\KL{Q}{P}/m$).
Since this is commonly the case in modern deep learning, we expect this bound to be tighter than others in this regime.

By utilizing the Theorem \ref{thm:SeegerBound} in the first step of the proof in section \ref{sect:MetaBoundProof} we can get a tighter bound for meta-learning:
\begin{gather}  \label{eq:MetaBoundSeeger}
\mathbb{P}_{(\Dcal_i, m_i) \sim\tau, S_i \sim  \Dcal_i^{m_i},i=1,...,n} \bigg\{  \er[\Qcal, \tau] \leq 
\frac{1}{n} \sum_{i=1}^{n} \bigg[
\Excpt{P}{\Qcal}  \erhati{Q_i(S_i,P),S_i}{i}    \\
+ 2 \varepsilon_i + \sqrt{2 \varepsilon_i \erhat[Q(S_i,P), S_i]} \bigg]  \nonumber 
+ \sqrt{\frac{1}{2(n-1)} \pth{\KL{\Qcal}{\Pcal} + \log \frac{2n}{\delta}}} 
, \forall \Qcal \bigg\} \geq 1 - \delta \nonumber, 
\end{gather}
where, $\varepsilon_i$ is defined in (\ref{SeegerIntraTask}) (and $\delta_i \triangleq \frac{\delta}{2n}$).

Finally we note that more recent works presented possibly tighter PAC-Bayes bounds by taking into account the empirical variance \citep{tolstikhin2013pac}, or by specializing the bound to deep  neural networks \citep{neyshabur2017pac} or by using more general divergences than the KL divergence \citep{alquier2018simpler}. However, we leave the incorporation of these bounds for future work.

\subsection{Hierarchical Variational Bayes} \label{sect:VariationalBayes}
In this section we show how  the variational inference method, used in a hierarchical Bayesian framework, can lead to a learning objective similar to the one obtained using PAC-Bayesian analysis. While the material here is not new (see, for example, \citep{blei2003latent, zhang2008flexible, edwards2016towards}), we present it for completeness.	
In the Bayesian framework one assume a probabilistic model with unknown (latent) variables, but with known prior distribution.
Given the observed data,  the aim  is to infer the posterior distribution over those variables using Bayes rule.
However, obtaining the posterior is often intractable.  Variational methods solve this problem by finding an approximate posterior.

In our case we observe the data sets of $n$ tasks $S_1,..,S_n$. Each $S_i$ is composed of $m_i$ samples $S_i=\braces{z_1,...,z_{m_i}}$.
As common in Hierarchal Bayesian methods  \citep{blei2003latent, zhang2008flexible, edwards2016towards} we assume a hierarchical model with shared random variable $\psi$  and task-specific random variables  $w_i,i=1,...,n$  (see Figure \ref{fig:graphical_model}). 

We make the following assumptions:
\begin{itemize}
	\item Known prior distribution over $\psi$,  $\Pcal(\psi)$.
	\item  Given  $\psi$, the pairs $\braces{(w_i,S_i),i=1,...,n}$ are mutually independent.
	\item $S_i$  is independent of $\psi$ given $w_i$, i.e,  $p(S_i|w_i, \psi) = p(S_i|w_i)$.
	\item Given $w_i$, the samples  $S_i=\braces{z_1,...,z_{m_i}}$ are independent, i.e, $p(S_i|w_i) = \prod_{z\in S_i} p(z|w_i)$.
	\item  Known likelihood function $p(z|w_i)$
	\footnote{The log-likelihood is analogous to the loss function PAC-Bayesian analysis \citep{germain2016pac}.}.
	\item  Known prior distribution over $w_i$ conditioned on $\psi$, $p(w_i| \psi)$.
\end{itemize}

\begin{figure}
	\centering
	\includegraphics[width=0.3\textwidth]{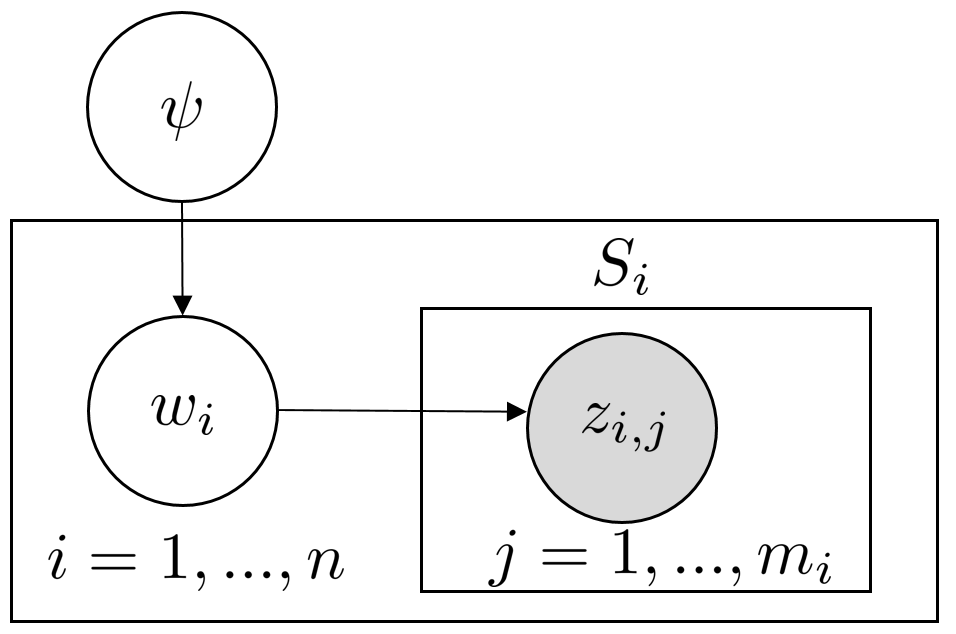}
	\caption{Graphical model of the framework: the \textit{circle} nodes denote random variables, \textit{shaded} nodes denote observed variables and \textit{plates} indicate replication.} 	
	\label{fig:graphical_model}
\end{figure}

The posterior over the latent variables can be written as 
\begin{gather}
p(\psi,w_1,...,w_n|S_1,..,S_n)	=  p(\psi|S_1,..,S_n ) p(w_1,...,w_n|\psi, S_1,..,S_n) \nonumber  \\
=  p(\psi|S_1,..,S_n ) \prod_{i=1}^{n} p(w_i|\psi, S_i), \label{eq:SimplePosterior}
\end{gather}
where the first equality stems from the conditional probability definition and the second equality from the conditional independence assumption.

Using Bayes rule and  the assumptions we have,
\begin{gather}
p(\psi|S_1,..,S_n ) = \frac{  p(S_1,..,S_n | \psi) \Pcal(\psi)}{p(S_1,..,S_n)}  = \frac{\prod_{i=1}^{n} p(S_i|\psi)\Pcal(\psi)}{p(S_1,..,S_n)}, \label{eq:P_posterior}
\end{gather}
\begin{gather}
p(w_i|\psi, S_i) = \frac{p(S_i| w_i, \psi)  p(w_i| \psi) }{p(S_i|\psi)}. \label{eq:w_posterior}
\end{gather}

Obtaining the exact posterior is intractable.
Instead, we will obtain an approximate solution using the following  family of distributions,
\begin{gather}
q(\psi, w_1,...,w_n) =  \Qcal_\theta(\psi)  \prod_{i=1}^{n} Q_{\phi_i}(w_i), \label{eq:VariationlPosterior}
\end{gather}
where $_\theta$ and $\phi_i$ are unknown parameters.
To obtain the best approximation we will solve the following optimization problem
\begin{gather*}
\argmin_{\theta, \phi_1,...,\phi_n} \KL{q(\psi,w_1,...,w_n)}{p(\psi,w_1,...,w_n|S_1,..,S_n)}.
\end{gather*}
Using \refeq{eq:VariationlPosterior} and \refeq{eq:SimplePosterior}, the optimization problem can be reformulated in an equivalent form
\begin{align*}
&\argmin_{\theta, \phi_1,...,\phi_n} \KL{\Qcal_\theta(\psi) \prod_{i=1}^{n} Q_{\phi_i}(w_i) }{p(\psi|S_1,..,S_n) \prod_{i=1}^{n} p(w_i|\psi, S_i)} \\
=&\argmin_{\theta, \phi_1,...,\phi_n}  \Excpt{\psi}{\Qcal_\theta} \Excpt{w_i}{Q_{\phi_i},i=1,..,n}
\brk{ \log \Qcal_\theta(\psi)  + \sum_{i=1}^{n} \log  Q_{\phi_i}(w_i) - \log p(\psi|S_1,..,S_n) -  \sum_{i=1}^{n} \log  p(w_i|\psi, S_i)}.
\end{align*}
Plugging \refeq{eq:P_posterior} we get
\begin{gather*}
\argmin_{\theta, \phi_1,...,\phi_n}  \Excpt{\psi}{\Qcal_\theta} \Excpt{w_i}{Q_{\phi_i},i=1,..,n}
\log \Qcal_\theta(\psi)  + \sum_{i=1}^{n} \log  Q_{\phi_i}(w_i) - \log \frac{\prod_{i=1}^{n} p(S_i|\psi)\Pcal(\psi)}{p(S_1,..,S_n)} -  \sum_{i=1}^{n}  \log  p(w_i|\psi, S_i) .
\end{gather*}
Rearranging and omitting terms independent of the optimization parameters  we get
\begin{align*}
&\argmin_{\theta, \phi_1,...,\phi_n}  \Excpt{\psi}{\Qcal_\theta}  \log \frac{\Qcal_\theta(\psi)}{\Pcal(\psi)} +
\Excpt{\psi}{\Qcal_\theta}   \sum_{i=1}^{n} \brk{ \Excpt{w_i}{Q_{\phi_i}}  \log  \frac {Q_{\phi_i}(w_i)}{p(w_i|\psi, S_i)} -\log p(S_i|\psi)} \\
=&\argmin_{\theta, \phi_1,...,\phi_n}  \KL{\Qcal_\theta}{\Pcal}  +
\Excpt{\psi}{\Qcal_\theta}   \sum_{i=1}^{n} \brk{ \Excpt{w_i}{Q_{\phi_i}}  \log  \frac {Q_{\phi_i}(w_i)}{p(w_i|\psi, S_i)} -\log p(S_i|\psi)}.
\end{align*}
Using \refeq{eq:w_posterior} we can re-write the term inside the sum 
\begin{align*}
&\Excpt{w_i}{Q_{\phi_i}}  \braces{\log  {Q_{\phi_i}(w_i)} - \log{p(w_i|\psi, S_i)}} -\log p(S_i| \psi) \\ 
=&\Excpt{w_i}{Q_{\phi_i}}  \braces{\log  {Q_{\phi_i}(w_i)}  - \log p(S_i| w_i, \psi)  - \log p(w_i|  \psi) + \log p(S_i| \psi)} - \log p(S_i| \psi) \\
=&\KL{Q_{\phi_i}}{p(w_i| P) } - \log p(S_i| w_i, \psi) .
\end{align*}
According to the assumptions we have $p(S_i| w_i, \psi) =  p(S_i| w_i) = \prod_{z \in S_i} p(z | w_i)$.
Finally we can write a simpler form for the optimization objective
\begin{equation} \label{eq:VariationalObjective}
\argmin_{\theta, \phi_1,...,\phi_n}  \Excpt{\psi}{\Qcal_\theta}  \sum_{i=1}^{n}\brk{\Excpt{w_i}{Q_{\phi_i}} 	\sum_{z \in S_i}-\log  p(z | w_i)  + \KL{Q_{\phi_i}}{p(w_i|\psi)}}    + \KL{\Qcal_\theta}{\Pcal}.
\end{equation}
The resulting learning objective is similar to the meta-learning generalization bound develop in our work, and indeed the experimental results are similar (see section \ref{sect:Results}). However, our algorithm is derived from a bound and is not formulated within a Bayesian framework.


\subsection{Pseudo Code}  \label{sect:code}

\begin{algorithm}[H] 
	\caption{MLAP algorithm, meta-training phase (learning-to-learn)}
	\label{algorithm:MetaTrain}
	\begin{algorithmic}
	\STATE {\bfseries Input:} {Data sets  of observed tasks: $S_1,...,S_n$.}
	\STATE {\bfseries Output:} {Learned prior parameters $\theta$.}
	\STATE {\bfseries Initialize:}\\
		\STATE  $\theta=\pth{\mu_P,\rho_P}\in \mathbb{R}^d \times\mathbb{R}^d $.
		\STATE  $\phi_i=\pth{\mu_i,\rho_i}\in \mathbb{R}^d \times\mathbb{R}^d, \quad \mathrm{for} \quad  i=1,...,n $.
	\WHILE {\textit{not done} }
		\FOR{each task  $i \in \braces{1,..n} \footnotemark$}
			\STATE Sample a random mini-batch from the data $S'_i \subset S_i$. \\
			\STATE Approximate  $J_i(\theta, \phi_i)$ (\ref{eq:TaskObj}) using $S'_i$ and averaging Monte-Carlo draws.\\  
		\ENDFOR 
		\STATE $J \leftarrow \frac{1}{n}\sum_{i \in \braces{1,..n}} J_i(\theta, \phi_i)  + \Upsilon(\theta)$.\\
		\STATE Evaluate the gradient of $J$ w.r.t $\braces{\theta, \phi_1,...,\phi_n}$ using backpropagation.\\
		\STATE Take an optimization step.
	\ENDWHILE
	\end{algorithmic}
\end{algorithm}

\footnotetext{For implementation considerations, when training with a large number of tasks we can sample a subset of tasks in each iteration (``meta min-batch" )  to estimate $J$.}

\begin{algorithm}[H]
	\caption{MLAP algorithm, meta-testing phase (learning a new task).}
	\label{algorithm:MetaTest}
	\begin{algorithmic}
	\STATE {\bfseries Input:} {Data set of a new task, $S$, and prior parameters, $\theta$.}
	\STATE {\bfseries Output:} {Posterior parameters $\phi'$ which solve the new task.}
	\STATE {\bfseries Initialize:}\\
		\STATE  $\phi' \leftarrow \theta$.
	\WHILE {\textit{not done}}
		\STATE Sample a random mini-batch from the data $S' \subset S$. \\
		\STATE Approximate the empirical loss $J$ (\ref{eq:TaskObj}) using $S'$ and averaging Monte-Carlo draws.	\\
		\STATE Evaluate the gradient of $J$ w.r.t $\phi'$ using backpropagation.\\
		\STATE Take an optimization step.
	\ENDWHILE	
	\end{algorithmic}
\end{algorithm}

\subsection{Classification Example Implementation Details} \label{sect:ImplemetDetails}

The network architecture used for the permuted-labels experiment is a small CNN with $2$ convolutional-layers of $10$ and $20$ filters, each with $5 \times 5$ kernels, a hidden linear layer with $50$ units and a linear output layer. Each convolutional layer is followed by max pooling operation with kernel of size 2. Dropout with $p=0.5$ is performed before the output layer.
In both networks we use ELU \citep{clevert2015fast} (with $\alpha=1$) as an activation function.
Both phases of the MLAP algorithm (algorithms \ref{algorithm:MetaTrain} and \ref{algorithm:MetaTest}) ran for $200$ epochs, with batches of $128$ samples in each task.
We take only one Monte-Carlo sample of the stochastic network output in each step.
As optimizer we used ADAM  \citep{kingma2014adam} with learning rate of $10^{-3}$.
The means of the weights ($\mu$ parameters) are initialized randomly with the Glorot method \citep{glorot2010understanding}, while the log-var of the weights ($\rho$ parameters) are initialized by $\normal{-10,0.1^2}$. 
The hyper-prior and hyper-posterior parameters are $\kappa_\Pcal=2000$ and $\kappa_\Qcal=0.001$ respectively and the confidence parameter  was chosen to be $\delta=0.1$ . To evaluate the trained network we used the maximum of the posterior for inference (i.e.\ we use only the means the weights)
\footnote{Classifying using the the majority vote of several runs gave similar results in this experiment.}.

\paragraph{MAML implementation details}
We report the best results obtained with all combinations of the following representative hyper-parameters: 
$1$-$3$ gradient steps in meta-training,  $1$-$20$ gradient steps in meta-testing, $300$ iterations and $\alpha \in \braces{0.01,0.1,0.4}$.
The best results for MAML were obtained $\alpha=0.01$, $2$ gradient steps in meta-training and $18$ in meta-testing.


\subsection{Visual Illustration in a Toy Example} \label{sect:ToyExample}

To illustrate the setup visually, we will consider a simple toy example of a 2D estimation problem. In each task, the goal is to estimate the mean of the data generating distribution.
In this setup, the samples $z$ are vectors in $\mathbb{R}^2$. The hypothesis class is a the set of 2D vectors, $h \in \mathbb{R}^2$.
As a loss function we will use the Euclidean distance, $\ell(h,z) \triangleq \norm{h-z}_2^2$.
We artificially create the data of each task by generating 50 samples from the appropriate distribution: $\normal{(2,1)^{\top},0.1^2 I_{2\times 2}}$ in task 1, and $\normal{(4,1)^{\top},0.1^2 I_{2\times 2}}$ in task 2.
The prior and posteriors are 2D factorized Gaussian distributions, $P \triangleq \normal{\mu_P, \diag(\sigma_P^2)}$ and $Q_i \triangleq \normal{\mu_i, \diag(\sigma_i^2)}, i=1,2$.

We run Algorithm \ref{algorithm:MetaTrain} (meta-training) with complexity terms according to Theorem \ref{thm:OriginalPacBayes}. As seen in Figure \ref{fig:Toy},  the learned prior (namely, the prior learned from the two tasks) and single-task posteriors can be understood intuitively. First, the posteriors are located close to the ground truth means of each task, with relatively small uncertainty covariance. 
Second, the learned prior is located in the middle between the two posteriors, and its covariance is larger in the first dimension.
This is intuitively reasonable since the prior learned that tasks are likely to have values of around 1 in  dimension 2 and values around 3 in the dimension 1, but with larger variance. Thus, new similar tasks can be learned using this prior with fewer samples.

\begin{figure}[H]
\vskip 0.2in
\begin{center}
	\includegraphics[width=0.4\textwidth]{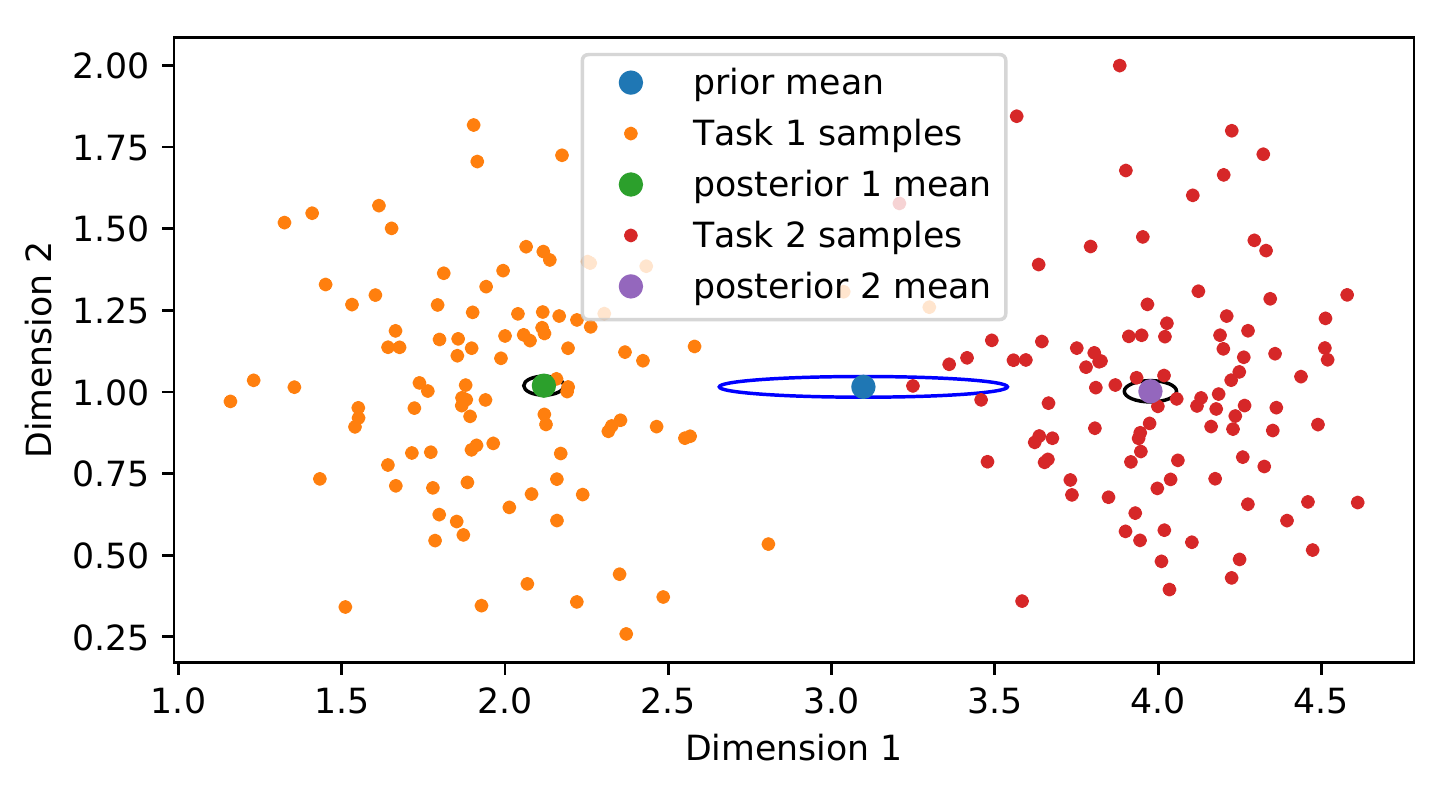}
	
	\caption{\textbf{Toy example:} 	the orange and red dots are the samples of task $1$ and $2$, respectively, and  the green and purple dots are the means of the posteriors of task $1$ and $2$, respectively. 
		The mean of the prior is a blue dot.	
		The ellipse around each distribution's mean represents the covariance matrix.
	}
\vskip -0.2in
\end{center}
\label{fig:Toy}
\end{figure}


\subsection{Technical Lemmas}

	\begin{Lemma} \label{Lem:Union}
		Let $\braces{E_i}_{i=1}^n$ be a set of events, which satisfy $\Prob(E_i) \geq 1- \delta_i$, with some $\delta_i \geq 0, i=1,...,n$.
		Then, $\Prob(\Intersect{i=1}{n} E_i) \geq 1 - \sum_{i=1}^{n} \delta_i$.
	\end{Lemma}

\begin{proof}
	First, note that 
	\begin{equation*}
\Prob(\Intersect{i=1}{n} E_i)  = 1 - \Prob(\Union{i=1}{n} E_i^C)  ,
	\end{equation*}
 where $E_i^C$ is the complementary event of $E_i$.

Using the union bound we have 
\begin{equation*}
\Prob(\Union{i=1}{n} E_i^C)   \leq  \sum_{i=1}^{n} \Prob(E_i^C)  
=\sum_{i=1}^{n} (1 -  \Prob(E_i)).
\end{equation*}

Therefore we have, 
\begin{equation*}
\Prob(\Intersect{i=1}{n} E_i) \geq 1 - \sum_{i=1}^{n} (1 -  \Prob(E_i)) 
 \geq  1 - \sum_{i=1}^{n} (1 -  (1-\delta_i)) = 1 - \sum_{i=1}^{n} \delta_i.
\end{equation*}

\end{proof}
\putbib
\end{bibunit}

\end{document}